\documentclass[12pt]{article}

\newif\ifpreprint   \preprinttrue
\newif\ifbiblatex   \biblatextrue

\usepackage{amsthm,amsmath,amssymb,mathtools}
\usepackage{graphicx,booktabs}
\usepackage[dvipsnames]{xcolor}
\usepackage{microtype}

\ifpreprint
  \usepackage[margin=1in]{geometry}
  \usepackage{authblk}
  \usepackage{newtxtext,newtxmath}   
\else
  \providecommand{\affil}[2][]{}     
\fi

\ifbiblatex
  \usepackage[style=alphabetic, maxbibnames=4]{biblatex}
\fi

\usepackage{my-macros}

\ifpreprint
  \usepackage[colorlinks=true,linkcolor=Maroon,
    urlcolor=MidnightBlue,citecolor=MidnightBlue]{hyperref}
\else
  \usepackage[colorlinks=false]{hyperref}
\fi
\usepackage[capitalise,noabbrev]{cleveref}

\providecommand{\keywords}[1]{\small\textbf{\textit{Keywords---}} #1}

\title{Distributional Shrinkage I: Universal Denoiser \\Beyond Tweedie's Formula\thanks{Previously circulated as ``Distributional Shrinkage I: Universal Denoisers in Multi-Dimensions'' (\href{https://arxiv.org/abs/2511.09500v1}{arXiv:2511.09500v1}).}}

\author{Tengyuan Liang}
\affil{The University of Chicago}

\date{\today}

\begin{document}
\maketitle

\begin{abstract}
We study the problem of denoising when only the noise level is known, not the noise distribution. Independent noise $Z$ corrupts a signal $X$, yielding the observation $Y = X + \sigma Z$ with known $\sigma \in (0,1)$. We propose \emph{universal} denoisers, agnostic to both signal and noise distributions, that recover the signal distribution $P_X$ from $P_Y$. When the focus is on distributional recovery of $P_X$ rather than on individual realizations of $X$, our denoisers achieve order-of-magnitude improvements over the Bayes-optimal denoiser derived from Tweedie's formula, which achieves $O(\sigma^2)$ accuracy. They shrink $P_Y$ toward $P_X$ with $O(\sigma^4)$ and $O(\sigma^6)$ accuracy in matching generalized moments and densities. Drawing on optimal transport theory, our denoisers approximate the Monge--Amp\`ere equation with higher-order accuracy and can be implemented efficiently via score matching.

Let $q$ denote the density of $P_Y$. For distributional denoising, we propose replacing the Bayes-optimal denoiser,
$$\mathbf{T}^*(y) = y + \sigma^2 \nabla \log q(y),$$
with denoisers exhibiting less-aggressive distributional shrinkage,
$$\mathbf{T}_1(y) = y + \frac{\sigma^2}{2} \nabla \log q(y),$$
$$\mathbf{T}_2(y) = y + \frac{\sigma^2}{2} \nabla \log q(y) - \frac{\sigma^4}{8} \nabla \!\left( \frac{1}{2} \| \nabla \log q(y) \|^2 + \nabla \cdot \nabla \log q(y) \right)\!.$$
\end{abstract}




\section{Introduction}
\label{sec:intro}

We revisit the classic denoising problem in multi-dimensions with $d \in \mathbb{N}^+$. Let $X, Z \in \mathbb{R}^d$ be two independent random variables, where $X$ is the signal of interest drawn from an unknown distribution $P_X$, and $Z$ is the noise with distribution $P_Z$ that is symmetric around zero and also unknown. We only observe the noisy measurement $Y$, with an additive noise level $\sigma \in (0, 1)$ known a priori:
\begin{align*}
Y = X + \sigma Z \;.
\end{align*}
The goal is to recover the underlying signal distribution $P_X$ from the distribution of noisy measurements $P_Y$. We emphasize that our goal is not to estimate each individual $X$, but rather to recover its probability distribution. We aim to construct \emph{universal denoising maps} $\mathbf{T}: \mathbb{R}^d \rightarrow \mathbb{R}^d$, such that the push-forward distribution $\mathbf{T} \sharp P_Y$ closely matches $P_X$ with high accuracy, which applies across \emph{a broad range of signal distributions $P_X$ and noise distributions $P_Z$}.

We work in a setting where the exact distribution of $Z$ is unknown but meets some mild moment conditions. Specifically, we do not assume Gaussianity or independence across coordinates for $Z$. This universality requirement is crucial in practice because the noise distribution is often unknown and may not be Gaussian. Our setting departs from the traditional empirical Bayes literature \cite{robbins1964empirical,efron1973stein,efron2016empirical}, which frequently assumes Gaussian noise or exponential families, and from the deconvolution literature \cite{carroll1988optimal,fan1991optimal,efron2016empirical}, which often presumes a known noise distribution.

One of the main messages we deliver in this paper is that if the goal is not to minimize the mean squared error of estimating individual realizations of $X$, but rather to recover the entire distribution $P_X$ accurately, then one should replace the Bayes-optimal denoiser, often limited to the Gaussian noise setting,
\begin{align*}
\mathbf{T}^*(y) = y + \sigma^2 \nabla \log q(y) \;, \quad\text{where $q:\mathbb{R}^d \rightarrow \mathbb{R}$ is the density function of $P_Y$} \;,
\end{align*}
by new first- and second-order optimal denoisers
\begin{align}
\mathbf{T}_1(y) &= y + \frac{\sigma^2}{2} \nabla \log q(y) \;, \label{eqn:T1}\\
\mathbf{T}_2(y) &= y + \frac{\sigma^2}{2} \nabla \log q(y) - \frac{\sigma^4}{8} \nabla \left( \frac{1}{2} \| \nabla \log q(y) \|^2 + \nabla \cdot \nabla \log q(y) \right) \;. \label{eqn:T2}
\end{align}
These denoisers achieve $O(\sigma^4)$ and $O(\sigma^6)$ accuracy in distributional matching, respectively, representing an order-of-magnitude improvement over the classical Bayes-optimal shrinkage $\mathbf{T}^*(y)$. These denoisers are \textbf{universal denoisers}, meaning they work under a host of smooth signal distributions $P_X$ and noise distributions $P_Z$ that satisfy mild moment conditions. These statements will be made precise in the following sections.

\begin{figure}[!ht]
	\centering
	\includegraphics[width=0.88\textwidth]{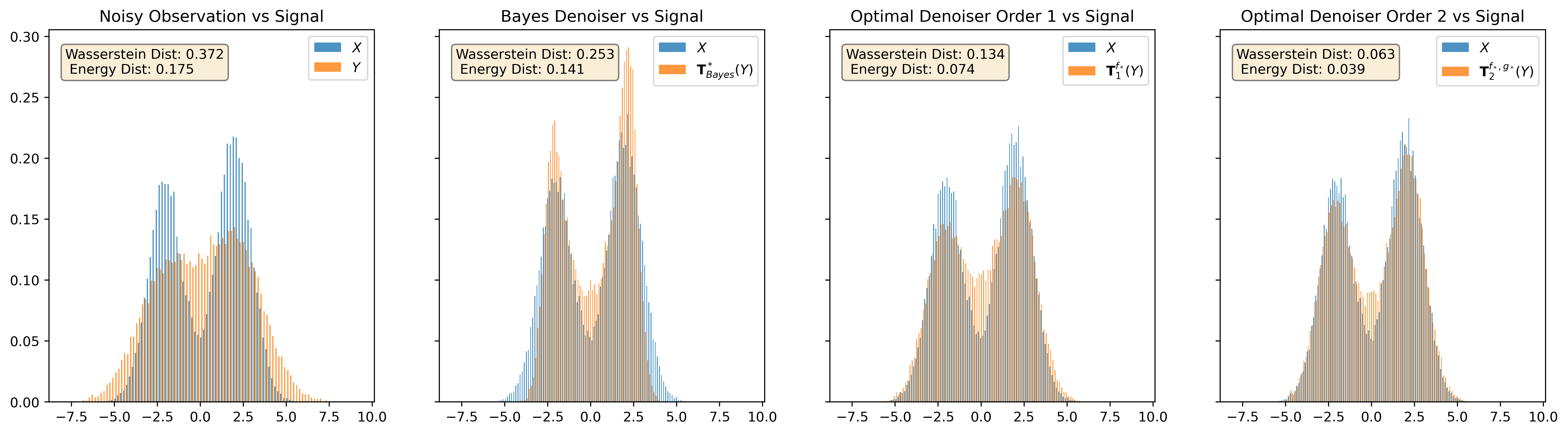}
	\includegraphics[width=0.88\textwidth]{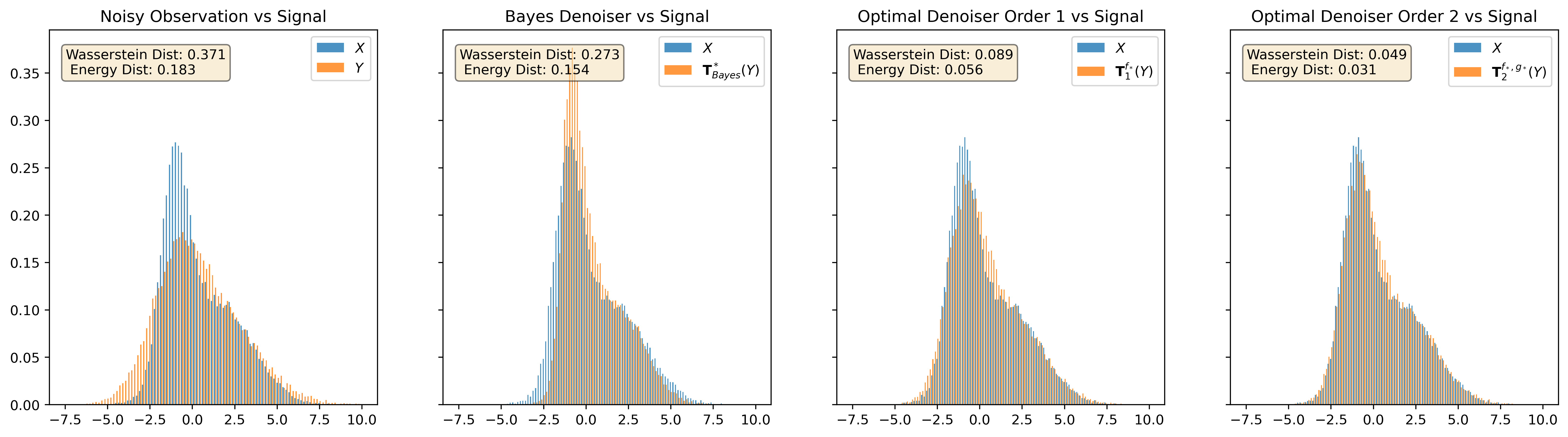}
	\caption{\scriptsize Comparison of shrinkage maps for $d=1$. We compare $\mathbf{T}^{\mathrm{id}} \sharp P_Y$ (no-shrinkage), $\mathbf{T}^* \sharp P_Y$ (Bayes-optimal shrinkage), and our proposed optimal-transport inspired shrinkage $\mathbf{T}_1^{f_*} \sharp P_Y$ and $\mathbf{T}_2^{f_*, g_*} \sharp P_Y$  against the true signal distribution $P_X$, from left to right. Each row corresponds to a different signal distribution $P_X$, chosen as a mixture of Gaussians. Here, the noise parameter $\eta = \sigma^2/2$ is chosen as 0.9. We report the Wasserstein \cite{villani2008optimal} and energy \cite{szekely2013energy} distances between the denoised and true signal distributions in the top-left corner of each subplot. As seen, the Bayes denoiser overly shrinks the distribution, whereas the no-shrinkage denoiser spreads it. The proposed denoisers $\mathbf{T}_1^{f_*}$ and $\mathbf{T}_2^{f_*, g_*}$ achieve improvement in distributional matching.}
	\label{fig:intro}
\end{figure}

\paragraph{Data-level Shrinkage}
A traditional approach to denoising is to construct a shrinkage map $\mathbf{T}: \mathbb{R}^d \rightarrow \mathbb{R}^d$ such that, at the data level, $\mathbf{T}(Y)$ matches $X$ as closely as possible. A well-known result, called Eddington-Robbins-Tweedie's formula \cite{eddington1940correction, robbins1964empirical,efron1973stein}, states that when the noise $Z$ is an isotropic Gaussian, i.e., $Z \sim \mathcal{N}(0, I_d)$, the optimal denoiser $\mathbf{T}^*$, in terms of mean squared error (MSE), $\E[ \| X - \mathbf{T}(Y) \|^2]$, is
\begin{align}
\label{eqn:empirical-bayes-shrinkage}
\mathbf{T}^*(y) := \E[X|Y = y] = y + \sigma^2 \nabla \log q(y) \;.
\end{align}
This identity lies at the heart of several well-known denoising techniques, from the traditional empirical Bayes methods to the more recent score-based diffusion models \cite{sohl2015deep,song2020score}. James-Stein shrinkage \cite{james1961estimation,brown1966admissibility} may be interpreted as a special case of Tweedie's formula, with $P_X$ being Gaussian \cite{efron1973stein} and estimated score function as $\widehat{\nabla \log q(y)} := -\tfrac{y}{\|Y\|^2/(d-2)}$.
\begin{align}
\mathbf{T}^{\mathrm{js}}(y) := \left( 1 - \frac{\sigma^2}{\|Y\|^2/(d-2)} \right) y \;.
\end{align}

This data-level shrinkage matches the first moment of the signal distribution $P_X$ on the distribution-level, yet shrinks the data too aggressively---namely $\E[\mathbf{T}^*(Y) \otimes \mathbf{T}^*(Y)] \prec  \E[X \otimes X]$---and matches the second moment only up to $\Theta(\sigma^2)$ accuracy when $\sigma$ is small. Specifically, we witness an over-shrinkage effect:
\begin{align*}
\text{Over-shrinkage}:~ \E[\mathbf{T}^*(Y) \otimes \mathbf{T}^*(Y)] -  \E[X \otimes X] =  - \sigma^2 (I_d + \sigma^2 \E[ \nabla^2 \log q(Y) ]) \preceq 0 \;.
\end{align*}
This over-shrinkage effect is known \cite{garcia2024new,jaffe2025constrained} and can be explained as follows: this data-level shrinkage focuses solely on minimizing the MSE, a (conditional) first-moment matching criterion. It does not aim to match higher-order moments of the distribution. Note that the trivial map $\mathbf{T}^{\mathrm{id}}(y) = y$ matches the first moment exactly but overshoots the second moment by $\Theta(\sigma^2)$:
\begin{align*}
\text{No-shrinkage}:~ \E[\mathbf{T}^{\mathrm{id}}(Y) \otimes \mathbf{T}^{\mathrm{id}}(Y)] -  \E[X \otimes X] =  \sigma^2 I_d \succeq 0 \;.
\end{align*}
Therefore, at the data level, the shrinkage estimator improves the MSE by only a constant factor, not by its order of magnitude in $\sigma^2$. In fact, the improvement in MSE is a constant ratio $\tfrac{ \sigma^2( d- \sigma^2 \E \| \log q(Y) \|^2 ) }{\sigma^2 d} \in (0, 1)$ with small $\sigma$. In this small-noise regime, we shall see shortly that taking a distribution-level viewpoint motivates denoisers that improve denoising performance by an order of magnitude.

Naturally, one wonders: is it possible to design denoisers such that $\mathbf{T} \sharp P_Y$ matches the second moment of $P_X$ up to $o(\sigma^2)$ accuracy? More ambitiously, is it possible to simultaneously match generalized moments with higher accuracy? Is it possible to construct universal denoising maps that work for a broad range of signal and noise distributions, extending well beyond the Gaussian noise setting? We show in this paper that the answer is affirmative.

To visually illustrate the over-shrinkage effect, we plot the distributions  $\mathbf{T}^{\mathrm{id}} \sharp P_Y$ (no-shrinkage) and $\mathbf{T}^* \sharp P_Y$ (Bayes-optimal shrinkage) against the signal distribution $P_X$ in the left two columns of Figures~\ref{fig:intro} (for $d=1$) and \ref{fig:intro-2d} (for $d=2$), where each row corresponds to a different signal distribution $P_X$. We observe that the denoised distribution $\mathbf{T}^* \sharp P_Y$ is overly concentrated compared to the true distribution $P_X$, while the no-shrinkage distribution $\mathbf{T}^{\mathrm{id}} \sharp P_Y$ is overly spread out. In contrast, the right two columns of Figures~\ref{fig:intro} and \ref{fig:intro-2d} plot the distribution shrinkage denoisers $\mathbf{T}_1$ in \eqref{eqn:T1} and $\mathbf{T}_2$ in \eqref{eqn:T2} proposed in this paper, which match the true distribution $P_X$ with markedly higher accuracy.

\paragraph{Distribution-level Shrinkage}
A different perspective is to view the denoiser $\mathbf{T}$ as a transport map that pushes forward the distribution of $Y$ to match that of $X$, i.e., $\mathbf{T} \sharp P_Y \approx P_X$. Consider the following less aggressive denoiser, which is motivated by optimal transport theory \cite{ambrosio2005gradient,villani2008optimal,liang2024denoising}:
\begin{align*}
 \mathbf{T}_1: y \mapsto y + \frac{\sigma^2}{2} \nabla \log q(y) \;,
\end{align*}
Compared with \eqref{eqn:empirical-bayes-shrinkage}, this shrinkage is less aggressive; it is the middle-point $$\mathbf{T}_1 = (\mathbf{T}^* + \mathbf{T}^{\mathrm{id}})/2 \;.$$ In the Gaussian setting, one can easily verify that $\mathbf{T}_1$ matches the second moment of $P_X$ up to $O(\sigma^4)$ accuracy, an order of magnitude improvement compared to Bayes-optimal shrinkage $\mathbf{T}^*(y)$ given by the Tweedie's formula,
\begin{align*}
\text{Optimal Shrinkage}:~ \big| \E[\mathbf{T}_1(Y) \otimes \mathbf{T}_1(Y)] -  \E[X \otimes X] \big| = O(\sigma^4) \;.
\end{align*}

Surprisingly, as we shall prove later, under mild conditions on the signal and noise distributions, this optimal transport-inspired shrinkage $\mathbf{T}_1$ achieves $O(\sigma^4)$ for a wide range of bounded smooth test functions $m: \mathbb{R}^d \rightarrow \mathbb{R}$,
\begin{align*}
\big| \E[ m(\mathbf{T}_1(Y)) ] - \E[ m(X) ] \big| = O(\sigma^4) \;.
\end{align*}
The third column in Figures~\ref{fig:intro} and \ref{fig:intro-2d} illustrates the denoised distributions $\mathbf{T}_1 \sharp P_Y$ against the signal distribution $P_X$, showing significant improvement over both the Bayes-optimal denoiser $\mathbf{T}^*$ and the no-shrinkage denoiser $\mathbf{T}^{\mathrm{id}}$, both visually and quantitatively.

\begin{figure}
	\centering
	\includegraphics[width=0.88\textwidth]{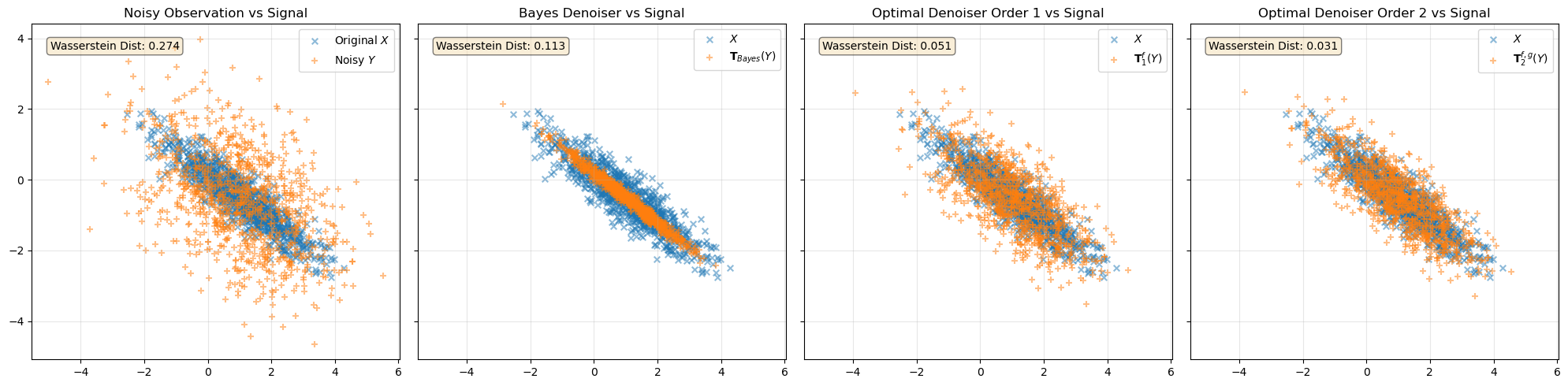}
	\caption{\scriptsize Comparison of shrinkage maps for $d=2$. We compare four denoisers as in Figure~\ref{fig:intro}, but with scatterplots. Here, the signal $X$ is Gaussian, and the noise parameter $\eta := \sigma^2/2$ is chosen as 0.5. We report the Wasserstein distance between the denoised and true signal distributions in the top-left corner of each subplot. Again, the Bayes denoiser shrinks the distribution too aggressively. The proposed denoisers $\mathbf{T}_1^{f_*}$ and $\mathbf{T}_2^{f_*, g_*}$ achieve improvement in distributional matching.}
	\label{fig:intro-2d}
\end{figure}

The discussions above motivate us to study distribution-level shrinkage: can we design denoisers that approximate the entire distribution of $P_X$ with even higher accuracy? In particular, for a host of smooth test functions $m: \mathbb{R}^d \rightarrow \mathbb{R}$, can we design denoisers $\mathbf{T}$ such that
\begin{align*}
\big| \E[ m(\mathbf{T}(Y)) ] - \E[ m(X) ] \big| = o(\sigma^4) \;?
\end{align*}

Similarly, can we say how the densities between the denoised distribution $\mathbf{T} \sharp P_Y$ and the signal distribution $P_X$ match? Let $p, q$ be the density functions of $P_X, P_Y$ respectively. It is known that the optimal transport map $\mathbf{T}^{\mathrm{ot}}$ \cite{brenier1991polar} that transports $P_Y$ to $P_X$ with minimum Wasserstein distance satisfies the Monge-Ampère equation \cite{caffarelli1992regularity,villani2008optimal,deb2025no} $p(\mathbf{T}^{\mathrm{ot}}(y)) \det( \nabla \mathbf{T}^{\mathrm{ot}}(y) ) - q(y) =0$.  Can we design $\mathbf{T}$ solely based on the knowledge of $P_Y$ and noise level $\sigma$, that approximately solves the Monge-Ampère equation up to higher-order accuracy, i.e.,
\begin{align*}
	| p(\mathbf{T}(y)) \det( \nabla \mathbf{T}(y) ) - q(y) | = o(\sigma^4)?
\end{align*}

In this paper, we examine the distribution-level shrinkage problem through the lens of optimal transport maps. We develop optimal denoisers that achieve errors of order $O(\sigma^4)$ and $O(\sigma^6)$, respectively, applicable universally to a broad class of signal distributions $P_X$ and noise distributions $P_Z$. The fourth column in Figures~\ref{fig:intro} and \ref{fig:intro-2d} showcases the effectiveness of the higher-order denoiser $\mathbf{T}_2$ in \eqref{eqn:T2}, which further improves the accuracy of distributional matching compared to $\mathbf{T}_1$.

Before diving into our theory and denoising equations, we conclude this section with a simple numerical simulation. We show in Figure~\ref{fig:order-of-magnitude-improvement} the order-of-magnitude improvement in denoising performance achieved by our proposed optimal first- and second-order denoisers $\mathbf{T}_1^{f_*}, \mathbf{T}_2^{f_*,g_*}$---defined in \eqref{eq:denoiser-first-order} and \eqref{eq:denoiser-second-order}---compared to the Bayes-optimal denoiser $\mathbf{T}^*$ and the no-shrinkage denoiser $\mathbf{T}^{\mathrm{id}}$. We report the second moment difference, the Wasserstein distance \cite{villani2008optimal}, and the energy distance \cite{szekely2013energy} between the denoised distribution and the true signal distribution. As shown, our proposed denoisers achieve lower errors across all three metrics, improving upon the Bayes-optimal denoiser by an order of magnitude, confirming our theoretical findings.

\begin{figure}
	\centering
	\includegraphics[width=0.32\textwidth]{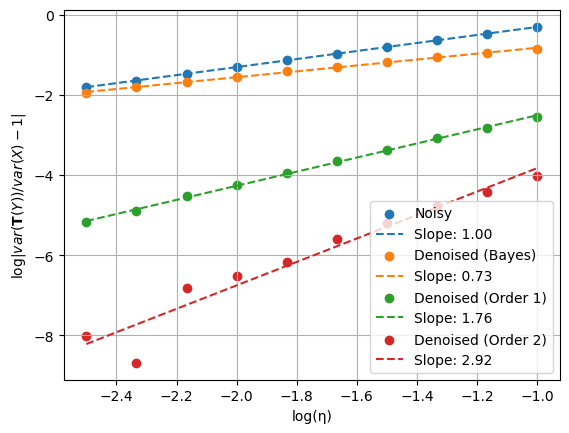}
	\includegraphics[width=0.32\textwidth]{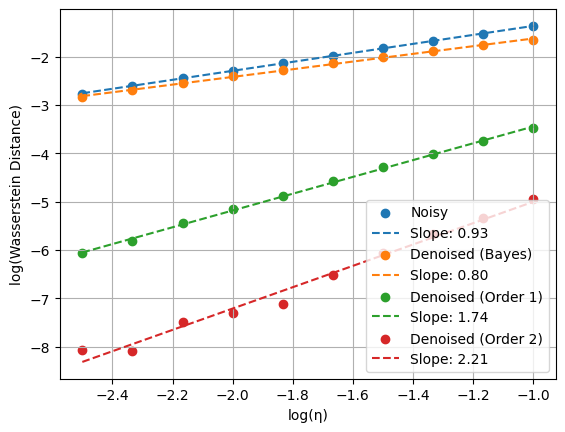}
	\includegraphics[width=0.32\textwidth]{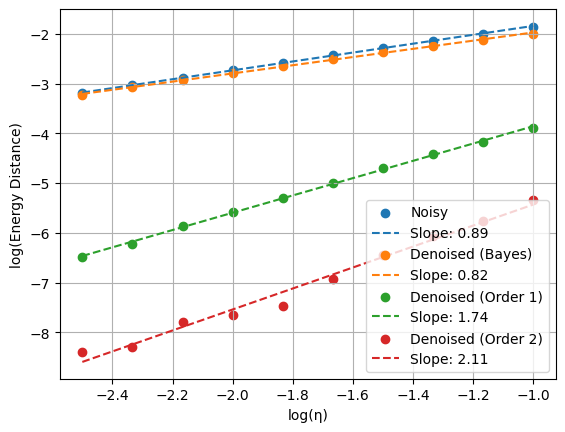}
	\caption{\scriptsize Comparing the order-of-magnitude in the denoising errors, varying the noise parameter $\eta = \sigma^2/2$. We compare four denoisers, the same as in Figure~\ref{fig:intro}: blue for $\mathbf{T}^{\mathrm{id}}$ (no-shrinkage), orange for $\mathbf{T}^*$ (Bayes-optimal shrinkage), green for $\mathbf{T}_1$ (our first-order denoiser), and red for $\mathbf{T}_2$ (our second-order denoiser). We report the variance error $\| \E[\mathbf{T}(Y) \otimes \mathbf{T}(Y)]- \E[X \otimes X]   \|/ \| \E[X \otimes X] \|$, the Wasserstein distance, and the energy distance between the denoised distribution and the true signal distribution. To numerically evaluate the errors between distributions, we sample $n=10^7$ and use the empirical distribution to approximate the population distribution. For each $\eta$ value, we compare the errors on the same n-samples $X \sim \cN(0, 1), Y\sim \cN(0, 1+2\eta)$ across four denoisers.}
	\label{fig:order-of-magnitude-improvement}
\end{figure}

\paragraph{On the Noise Level}
We remark that $\sigma \in (0, 1)$ can be assumed without loss of generality. If $\sigma \geq 1$, one can rescale $Y$ by $\tau : = \| \E[Y Y^\top] \|_{\mathrm{op}}^{1/2} > \sigma$ to ensure that the effective noise level $\widetilde{\sigma} := \sigma/\tau < 1$, as long as $X$ is not a point mass at $0$; rescaling $Y$ by $1.1\sigma$ also suffices as $\widetilde{\sigma}= 1/1.1$. Therefore, throughout the paper, we assume $\sigma \in (0, 1)$ and thus $\eta = \sigma^2/2 \in (0, 1/2)$ without loss of generality.

\paragraph{Notations}
Throughout the paper, we use bold uppercase letters $\mathbf{T}$ to denote maps from $\mathbb{R}^d$ to $\mathbb{R}^d$. We use $x, y, z$ to denote vectors in $\mathbb{R}^d$, and $X, Y, Z$ to denote random variables in $\mathbb{R}^d$. We use $P_X,  P_Y,  P_Z$ to denote probability distributions of random variables, and their density functions w.r.t. the Lebesgue measure are denoted as $p, q, \phi :\mathbb{R}^d \rightarrow \mathbb{R}$. When there is no ambiguity, we write $p = P_X$ and $q = P_Y$ for simplicity. For a probability distribution $P$, we use $\E_{X \sim P}[\cdot]$ to denote the expectation w.r.t. $X \sim P$. For a map $\mathbf{T}: \mathbb{R}^d \rightarrow \mathbb{R}^d$ and a distribution $P$, we use $\mathbf{T} \sharp P$ to denote the push-forward distribution of $P$ through the map $\mathbf{T}$.

We use standard notations for gradients and Hessians: for a scalar function $f: \mathbb{R}^d \rightarrow \mathbb{R}$, $\nabla f$ is the gradient vector in $\mathbb{R}^d$, $\nabla^2 f$ is the Hessian matrix in $\mathbb{R}^{d \times d}$, and $\Delta f = \mathrm{trace}(\nabla^2 f)$ is the Laplacian operator. For a vector-valued function $\mathbf{T}: \mathbb{R}^d \rightarrow \mathbb{R}^d$, $\nabla \mathbf{T}$ is the Jacobian matrix in $\mathbb{R}^{d \times d}$, and $\det(\nabla \mathbf{T})$ is its determinant. 
For two symmetric matrices $A, B \in \mathbb{R}^{d \times d}$, we write $A \preceq B$ if $B - A$ is positive semi-definite. $I_d \in \mathbb{R}^{d\times d}$ is reserved for the identity matrix. 
For a vector $x$ and a matrix $A$, we use $\| x\|_\infty := \max_i |x_i|$ and $\| A \|_{\infty} = \max_{i,j} |A_{ij}|$, and we use $\| x\|$ and $\| A\|$ for its Euclidean norm and Frobenius norm, respectively. $\langle x, y\rangle$ and $\langle A, B \rangle$ denote the Euclidean inner product between two vectors and the Frobenius inner product between two matrices, respectively.
We use standard big-O notation: for two functions $f, g: [0, 1] \rightarrow \mathbb{R}$, we write $f(\eta) = O(g(\eta))$ if there exists a constant $C > 0$ such that $|f(\eta) /  g(\eta)| \leq C$ for all sufficiently small $\eta > 0$.

\paragraph{Organization}
The paper is organized as follows. In Section~\ref{sec:related-work}, we discuss the relations to the existing literature. In Section~\ref{sec:denoising-ode}, we introduce the differential equations characterizing the optimal denoisers along with practical implementations, followed by assumptions on signal and noise distributions in Section~\ref{sec:assumptions}. In Section~\ref{sec:moment-matching}, we study the theoretical performance of the proposed denoisers in terms of generalized moment matching. In Section~\ref{sec:monge-ampere-optimality}, we establish the optimality of the proposed denoisers in terms of approximating the static Monge-Ampère equation with higher-order accuracy. In Section~\ref{sec:numerics}, we empirically evaluate the optimal first- and second-order denoisers and compare them with the Bayes-optimal denoisers. The proofs of the main theorems are presented in Section~\ref{sec:proofs-of-theorems}, while technical lemmas and their proofs are included in Section~\ref{sec:proofs-of-supporting-lemmas}.

\section{Relations to the Literature}
\label{sec:related-work}

The empirical Bayes framework, pioneered by Robbins~\cite{robbins1964empirical}, studies the problem of estimating an unknown signal from noisy measurements. Under Gaussian noise, Eddington-Robbins-Tweedie's formula \cite{eddington1940correction,robbins1964empirical,ignatiadis2025empirical} provides an explicit formula for the Bayes-optimal denoiser. When the signal is also Gaussian, \cite{efron1973stein} showed that the empirical Bayes shrinkage provides a natural alternative to the James-Stein shrinkage \cite{james1961estimation,brown1966admissibility}, which improves upon the maximum likelihood estimator in terms of mean squared error by a constant factor. \cite{efron1973stein} further discusses the possibility of relaxing both the Gaussian noise and Gaussian signal assumptions, but still considering the linear empirical Bayes rule to induce shrinkage and denoising. Since then, extensive research has been done. On the one hand, several papers consider general signal distributions, but with Gaussian noise and Tweedie's formula, for example, see \cite{jiang2009general} for nonparametric distribution and \cite{stephens2017false} for Gaussian mixture distributions. On the other end, moving away from the Gaussian noise assumption typically requires explicit knowledge of the noise distribution, as in deconvolution \cite{carroll1988optimal,fan1991optimal,efron2016empirical}, or implicit knowledge like replicates of the noise on the same data point \cite{ignatiadis2025empirical,saremi2019neural}. However, deviating from Gaussian noise typically requires replacing the denoising map from Tweedie's formula with more complicated forms, losing a universal shrinkage formula. Our work fills in this gap. We show that, if the goal is to match the entire distribution of the signal rather than minimizing mean squared error, it is possible to design universal denoisers that work across a wide range of signal and noise distributions. These denoising maps differ from the classical Tweedie's formula, and are agnostic to the exact noise distribution, requiring only mild moment conditions. 

We work in a regime with a constant noise level, where $\sigma$ is a known constant, but the exact noise distribution is unknown. In this case, exact deconvolution is intractable due to the lack of knowledge of the noise distribution. We design simple, universal denoising maps that depend only on the score function of the noisy measurement distribution $q = P_Y$. Our denoisers achieve $O(\sigma^4)$ and $O(\sigma^6)$ accuracy in distributional matching for a range of signal distribution $P_X$ and noise distribution $P_Z$, respectively, improving over the classical Bayes-optimal shrinkage.

Optimal transport theory \cite{villani2008optimal, ambrosio2005gradient} offers a new framework for comparing and transforming probability distributions. Taking an optimal transport perspective also opens new avenues for revisiting the denoising problem. Due to the variational formulation of the Fokker--Planck equation \cite{jordan1998variational}, adding small Gaussian noise to a distribution can be interpreted as gradient descent (with a small step size) in the space of probability measures with the entropy functional as the objective and Wasserstein distance as the geometry. Therefore, denoising can be viewed as the reverse process of this gradient descent; see Proposition 1 in \cite{liang2024denoising}. As a result, we can study denoising from the perspective of transport maps. Recent works have explored the use of optimal transport maps for denoising \cite{kim2023consistency,chen2023probability,liang2024denoising} and generative modeling \cite{hur2024reversible, liang2024denoising, deb2025no}. Closest to ours are \cite{garcia2024new,jaffe2025constrained,liang2025distributionalshrinkageiioptimal}. \cite{garcia2024new,jaffe2025constrained} also notice the over-shrinkage effect of the empirical Bayes denoiser. They suggest post-processing empirical Bayes methods to correct for distributional matching, assuming knowledge of $P_X$ or its moments. In contrast, our denoisers are universal and do not require any knowledge of $P_X$. \cite{liang2025distributionalshrinkageiioptimal} studies distributional shrinkage under Gaussian noise in a one-dimensional setting but derives a hierarchy of denoisers that leverage higher-order score functions. In contrast, our work studies the problem in a multidimensional setting and constructs first- and second-order denoisers that remain universal to a broad range of noise distributions beyond Gaussian.

Our approach leverages ideas from optimal transport to design denoisers that achieve distributional matching with higher accuracy. The accuracy is measured either by integral probability metrics \cite{muller1997integral, gretton2012kernel, liang2019estimating, liang2021well} induced by smooth test functions, and by approximations to the static Monge-Ampère equation \cite{caffarelli1992regularity,deb2025no}. The set of smooth test functions can be interpreted as generalized moments of the distribution \cite{hansen1982large}. The static Monge-Ampère equation characterizes the optimal transport map between two distributions \cite{caffarelli1992regularity,villani2008optimal}. Our work builds on these ideas to derive differential equations characterizing the optimal denoisers that achieve higher-order distributional matching. 

Score matching techniques \cite{hyvarinen2005estimation} aim to learn the score function $\nabla \log q(y)$ from data. These methods have been successfully applied in generative modeling~\cite{vincent2011connection, sohl2015deep, song2020score} and empirical Bayes \cite{ghosh2025stein}. \cite{saremi2019neural} connects score matching to denoising by learning data-driven empirical Bayes denoisers based on neural networks. There is a growing literature on both methodological and theoretical sides of score matching \cite{hyvarinen2007some, saremi2018deep, sriperumbudur2017density, block2020generative,koehler2022statistical, ghosh2025stein, feng2024optimal}, among others.  Our work complements these results by providing explicit optimal denoising maps that achieve higher-order distributional matching for a wide range of signal and noise distributions. Furthermore, our denoisers can be estimated from data using score matching techniques \cite{hyvarinen2005estimation}, as they are constructed based on the score function $\nabla \log q(y)$ and its derivatives. This connection enables the practical implementation of our denoisers in real-world applications \cite{hyvarinen2005estimation,saremi2019neural,song2020score} using automatic differentiation in modern machine learning frameworks. Curiously, our second-order optimal denoiser relies on the exact form as the functional objective minimized in score matching objective \cite{hyvarinen2005estimation}, or the celebrated Stein's unbiased risk estimate objective \cite{stein1981estimation,xie2012sure,ghosh2025stein}. 

Our distributional denoising theory in the small, constant noise regime $\eta = \sigma^2/2$ provides theoretical foundations for new denoisers that achieve higher-order accuracy (as a function of the discretized step size or noise parameter $\eta$). This theory sits at the core of modern diffusion-based generative models, where new denoisers can be employed in the backward process of diffusion models \cite{sohl2015deep,song2020score} to replace the stochastic backward diffusion process with a deterministic optimal denoising map \cite{chen2023probability,liang2024denoising} that achieves higher-order accuracy in denoising the diffusions.

\section{Denoisers, Signal and Noise}

\subsection{Differential Equations for Optimal Denoisers}
\label{sec:denoising-ode}
We introduce a noise parameter scaling $\eta = \sigma^2/2$, then define the first- and second-order shrinkage denoisers (in $\eta$) as local-to-identity maps
\begin{align*}
\mathbf{T}_1^{f}(y) &:= y + \eta \nabla f(y) \;,\\
\mathbf{T}_2^{f,g}(y) &:= y + \eta \nabla f(y) + \frac{\eta^2}{2} \nabla g(y) \;,
\end{align*}
where $f, g: \mathbb{R}^d \rightarrow \mathbb{R}$ are smooth functions to be determined that induce the shrinkage directions. Our goal is to find the optimal functions $f, g$ such that the denoised distributions $\mathbf{T}_1^{f} \sharp P_Y$ and $\mathbf{T}_2^{f,g} \sharp P_Y$ match the true signal distribution $P_X$ with high accuracy.

Recall that $q: \mathbb{R}^d \rightarrow \mathbb{R}$ is the density function of the noisy measurement distribution $P_Y$. 
It turns out the optimal denoisers in terms of distributional matching up to $O(\eta^2)$ and $O(\eta^3)$ satisfy the following system of differential equations,
\begin{align*}
\text{optimal first-order equation:} \quad q \nabla f &=  \nabla q \;,  \\
\text{optimal second-order equation:} \quad q \nabla g &=   \nabla^2 f \nabla q + \Delta q \nabla f - \nabla \Delta q  \;.
\end{align*}
As a result, the optimal denoisers are induced by the following choice of $f, g$,
\begin{align}
f_* &= \log q \;, \label{eq:denoiser-first-order}\\
g_* &= -\left[ \frac{1}{2} \| \nabla f_* \|^2 + \nabla \cdot \nabla f_* \right] = - \left[ \frac{1}{2} \| \nabla \log q \|^2 + \Delta \log q  \right]\;. \label{eq:denoiser-second-order}
\end{align}

These equations may appear mysterious at first glance. Their origin will be elucidated in Section~\ref{sec:monge-ampere-optimality}, where we study the Monge-Ampère equation that characterizes the optimal transport map between the noisy measurement distribution $P_Y$ and the true signal distribution $P_X$. The denoiser is optimal in that it approximately solves the static Monge-Ampère equation with accuracy $O(\eta^2)$ and $O(\eta^3)$, respectively. Notably, the optimal denoisers depend only on the noisy distribution $P_Y$ through its score function $\nabla \log q$ and its derivatives, and are agnostic to the exact noise distribution $P_Z$ and the signal distribution $P_X$. This implies that one can approximate the optimal transport map for denoising $P_Y$ to $P_X$, requiring only knowledge of $P_Y$ and the noise level $\sigma$.

Our denoisers can be learned from data using score matching techniques \cite{hyvarinen2005estimation}. Recall the Fisher divergence minimization approach for score matching, which solves
\begin{align}
\min_{\xi \in \mathcal{S}} \E_{Y \sim q}\frac{1}{2}\|\xi(Y) - \nabla \log q(Y) \|^2 = \min_{\xi: \mathbb{R}^d \rightarrow \mathbb{R}^d} \E_{Y \sim q} \left[ \frac{1}{2} \| \xi(Y) \|^2 +  \nabla \cdot \xi(Y) \right] + \mathrm{Const} \;, \label{eqn:fisher-divergence-minimization}
\end{align}
where $\mathcal{S}$ is a vector-valued function class for score functions $\xi: \mathbb{R}^d \rightarrow \mathbb{R}^d$. The optimal solution to the above optimization problem is exactly $\xi_* = \nabla f_* = \nabla \log q$, the gradient of the optimal first-order denoiser.

With an estimated score function $\widehat{\xi}$ that is smooth, we can directly plug it in to obtain data-driven first- and second-order denoisers
\begin{align}	
\widehat{\mathbf{T}}_1(y) &:= y + \eta \widehat{\xi}(y) \; \, \label{eqn:est-first-order-denoiser}\\
\widehat{\mathbf{T}}_2(y) &:= y + \eta \widehat{\xi}(y) - \frac{\eta^2}{2} \nabla \left[ \frac{1}{2} \| \widehat{\xi}(y) \|^2 +  \nabla \cdot \widehat{\xi}(y) \right] \;, \label{eqn:est-second-order-denoiser}
\end{align}
Note that the equation in \eqref{eqn:est-second-order-denoiser} precisely matches the functional form in the score matching objective \eqref{eqn:fisher-divergence-minimization}. Therefore, the denoisers can be obtained directly from the estimated score function $\widehat{\xi}$ and computed efficiently using automatic differentiation in modern machine learning frameworks. We implement this approach in our numerical experiments; see Section~\ref{sec:numerics} for details of our implementation.

Again, there has been a growing literature on the consistency and estimation accuracy of score matching from finite samples, see \cite{sriperumbudur2017density, block2020generative, koehler2022statistical, ghosh2025stein, feng2024optimal}. Our work complements these results. With an estimated score function $\widehat{\xi}$, we can directly plug in to obtain estimated denoisers $\widehat{\mathbf{T}}_1$ and $\widehat{\mathbf{T}}_2$ as in \eqref{eqn:est-second-order-denoiser}, designed to achieve higher-order accuracy in distributional denoising, a key task in modern diffusion-based generative models \cite{sohl2015deep,song2020score}. 

\subsection{Assumptions on the Signal and Noise}
\label{sec:assumptions}

In this section, we discuss the assumptions imposed on the signal distribution $P_X$ and the noise distribution $P_Z$, in preparation for the theoretical results in terms of denoising quality; Section~\ref{sec:moment-matching} discusses moment matching property for general moments, and Section~\ref{sec:monge-ampere-optimality} studies distributional matching in an affinity measure induced by Monge-Ampère equation. 

The first assumption is made on the signal density $p = P_X$ with respect to the Lebesgue measure. We require smoothness conditions on $p$. This is a standard assumption in nonparametric statistics, known as $k$-th order H\"{o}lder smoothness. We only need smoothness up to $k=4$ for the accuracy of the first-order denoiser, and up to $k=6$ for that of the second-order denoiser, independent of the dimension $d$.

\begin{assumption}[$k$-th Order Smoothness]
	\label{assump:density-smoothness}
	The density $p: \mathbb{R}^d \rightarrow \mathbb{R}_{\geq 0}$ satisfies
	\begin{align*}
	\| \nabla^k p \|_{\infty} := \sup_{x \in \mathbb{R}^d} \| \nabla^k p(x) \|_{\infty} < \infty \;,
	\end{align*}	
	where $\| \nabla^k p(x) \|_{\infty} := \max_{i_1, \ldots, i_k \in [d]} | \partial_{x_{i_1}} \cdots \partial_{x_{i_k}} p(x) |$.
\end{assumption}

The second assumption is made on the noise variable $Z \in \mathbb{R}^d$. We require certain moment conditions on $Z$. In particular, we do not assume Gaussianity nor independence across coordinates for $Z$. The first-order denoiser defined in \eqref{eq:denoiser-first-order} remains universal and agnostic to the exact noise distribution under Assumption~\ref{assump:noise-moments}(i), whereas the second-order denoiser defined in \eqref{eq:denoiser-second-order} requires both Assumption~\ref{assump:noise-moments}(i) \& (ii).

\begin{assumption}[Moments of Noise]
	\label{assump:noise-moments}
	The noise $Z \in \mathbb{R}^d$ is symmetric around $0$, namely, $Z \overset{d}{=} -Z$. For any $i,j,k,l \in [d]$:

	(i) $\E[Z_i Z_j] = \delta_{ij}$, and $\E[|Z_i|^4] < \infty$ ;

	(ii) $\E[Z_i Z_j Z_k Z_l] = \delta_{ij}\delta_{kl} + \delta_{ik}\delta_{jl} + \delta_{il}\delta_{jk}$, and $\E[|Z_i|^6] < \infty$.

\end{assumption}

As we shall see, for the theoretical result on the first-order denoiser, only Assumption~\ref{assump:noise-moments}(i) and Assumption~\ref{assump:density-smoothness} up to $k=4$, are needed. In particular, we only require the noise variable $Z$ to be uncorrelated, $\E[Z Z^\top]=I_d$, and to have bounded 4-th moments; no need to assume Gaussianity or independence across coordinates. In this sense, the first-order denoiser is universal to a large class of noise distributions.

For the theoretical result on the second-order denoiser, both Assumption~\ref{assump:noise-moments}(i) \& (ii) are needed, and Assumption~\ref{assump:density-smoothness} up to $k=6$. Admittedly, Assumption~\ref{assump:noise-moments}(ii) is more restrictive, as it requires the noise $Z$ to have Gaussian-like $4$-th moments. However, we note that this assumption remains significantly weaker than assuming Gaussianity. In this context, the second-order denoiser is less universal than the first-order denoiser but still agnostic to the specific form of noise distribution. The fourth-moment condition can be relaxed to allow for arbitrary $Z$ with i.i.d. coordinates with a finite fourth-moment, at the cost of a slightly more complex denoising equation, which will be discussed in Section~\ref{sec:monge-ampere-optimality}.

\section{Matching Generalized Moments}
\label{sec:moment-matching}
In this section, we study how the distribution $\mathbf{T} \sharp P_Y$ matches $P_X$ in terms of generalized moments.

Considering a host of test functions $m: \mathbb{R}^d \rightarrow \mathbb{R} \in \mathcal{M}$, we compare how the moments of $\mathbf{T} \sharp P_Y$ and $P_X$ match:
\begin{align}
\big| \E_{Y \sim P_Y} [ m(\mathbf{T}(Y)) ] - \E_{X \sim P_X} [ m(X) ] \big| \;.
\end{align}
Conceptually, $\mathcal{M}$ is a class of test functions \cite{muller1997integral}, or called the generalized moments \cite{hansen1982large}. When $\mathcal{M}$ is the class of all bounded measurable functions, then matching over this class reduces to $\mathbf{T} \sharp P_Y$ being close to $P_X$ in total variation distance. When $\mathcal{M}$ is the class of all Lipschitz functions, then these two distributions are close in Wasserstein distance. When $\mathcal{M}$ is the class of polynomials up to degree $k$, then these two distributions are close in the moments up to order $k$.

We investigate the performance of denoisers (induced by $f_*, g_*$) under an integral probability metric \cite{muller1997integral,gretton2012kernel,liang2019estimating} induced by a class of smooth test functions $\mathcal{M}$. 
Note that $f_*, g_*$ depend only on the density $q = P_Y$, and thus the following integrability assumption is imposed on derivatives of $q$ to ensure $f_*, g_*$ are well-behaved. This mild regularity assumption holds generally when the density $q$ is smooth and decays at infinity.

\begin{assumption}[Integrability of Denoisers]
	\label{assump:integrability}
	Recall the definition of $f_*, g_*$ in \eqref{eq:denoiser-first-order} and \eqref{eq:denoiser-second-order} that depends on $q = P_Y$.  The density $q$ satisfies, for any $\eta \in [0, c)$ for some constant $c > 0$:
	
	(i) $\int_{\mathbb{R}^d} \| \nabla f_*(y) \|^2 q(y) \dd y < \infty$;

	(ii) $\int_{\mathbb{R}^d} \| \nabla f_*(y) \|^3 + \| \nabla g_*(y) \|^3  q(y) \dd y < \infty$, and $\sup_{y \in \mathbb{R}^d} \| \nabla f_*(y)\|^2 q(y) < \infty$.

\end{assumption}

\subsection{First-Order Denoiser}

We first define the class of second-order smooth test functions $\mathcal{M}^2(\mathbb{R}^d)$, which will be used to analyze the performance of the first-order denoiser $\mathbf{T}_1^{f_*}$.

\begin{definition}[2-Smooth Test Functions]
	\label{def:2-smooth-test-functions}
	The class of 2-smooth test functions $\mathcal{M}^2(\mathbb{R}^d)$ is defined as
	\begin{equation*}
	\cM^{2}(\mathbb{R}^d) := \left\{ m: \mathbb{R}^d \rightarrow \mathbb{R} \;\middle|\;
	\begin{aligned}
	&\sup_{y \in \mathbb{R}^d} \| \nabla^2 m(y) \|_{\infty} < \infty, \quad m~\text{vanishes at infinity} \;, \\
	&\int_{\mathbb{R}^d} |m(y)| \dd y < \infty, \quad \int_{\mathbb{R}^d} \|\nabla m(y) \|_{\infty} \dd y < \infty \;.
	\end{aligned} 
	\right\} \;.
	\end{equation*}
\end{definition}

Some explanations are in order. First, the integrability conditions impose decay conditions on $m$ and its gradient $\nabla m$ at infinity. This rules out unbounded test functions that grow to infinity, such as polynomials. However, there is a host of test functions that satisfy these conditions, for example, compactly supported smooth functions. Second, when the signal distribution $P_X$ has a compact support or decays fast at infinity, it is natural only to consider test functions supported on a compact set $\Omega \subset \mathbb{R}^d$, with smooth extensions to the entire space $\mathbb{R}^d$ that vanish. Such a consideration is without loss of generality for bounded test functions, or unbounded test functions with a slow growth compared to the decay of $P_X$ at infinity. Third, let $\varepsilon>0$, one can consider $m$ to be any continuous function supported on a compact set $\Omega$ (including polynomials), then consider smoothed test functions $m * \varphi_{\varepsilon}$ convolved with a mollifier $\varphi_{\epsilon}$ \cite{evans2022partial}, say $\varphi_{\varepsilon}(x) = \varepsilon^{-d} \varphi(x/\varepsilon)$ with $$\varphi(x) = \begin{cases}
C \exp( -\frac{1}{1-\|x\|^2} ) & \|x\| < 1 \;,\\
0 &  \|x\| \geq 1 \;.
\end{cases}$$
to obtain a smooth approximation of $m$ that naturally satisfies all conditions in $\mathcal{M}^2(\mathbb{R}^d)$.

\begin{theorem}[Moment Matching: Optimal First-Order Denoiser]
	\label{thm:first-order}
	Under Assumption~\ref{assump:integrability}(i), we consider the first-order denoiser $\mathbf{T}_1^{f_*}(y)$ defined in \eqref{eq:denoiser-first-order},
	$$\mathbf{T}_1^{f_*}(y) = y + \frac{\sigma^2}{2} \nabla \log q(y)\;.$$
	Let $p = P_X$ satisfy Assumption~\ref{assump:density-smoothness} for up to $k=4$, and the noise $Z$ has bounded 4-th moments as in Assumption~\ref{assump:noise-moments}(i). We then have for $\forall m \in \mathcal{M}^2(\mathbb{R}^d)$, the $2$-smooth test function as in Definition~\ref{def:2-smooth-test-functions},
	\begin{align*}
	 \big| \E_{Y \sim P_Y} [ m(\mathbf{T}_1^{f_*}(Y)) ] - \E_{X \sim P_X} [ m(X) ] \big| \leq C_{m} \cdot \sigma^4 \;.
	\end{align*}
	Here, the universal constant $C_{m} > 0$ only depends on the test function $m$.
\end{theorem}

This theorem states that for a smooth test function $m \in \cM^2(\mathbb{R}^d)$, the first-order denoiser $\mathbf{T}_1^{f_*}$ matches the generalized moments (induced by $m$) of the signal distribution $P_X$ up to an error of order $O(\sigma^4)$. Such a result holds as long as the signal distribution $P_X$ is smooth enough (up to 4-th order), and the noise distribution $P_Z$ has bounded 4-th moments and is uncorrelated. In this sense, the first-order denoiser is universal to a broad class of signal and noise distributions; it is agnostic to exact forms of both $P_X$ and $P_Z$, allowing both to be general nonparametric distributions. Thus, we call them universal denoisers.

It is possible to state this theorem uniformly over a class of test functions $\mathcal{M}^2(\mathbb{R}^d)$, by imposing uniform bounds on the derivatives of $m \in \mathcal{M}^2(\mathbb{R}^d)$, and uniform integrability conditions. For simplicity, we present the result for each fixed test function $m$. This is because to apply the result to an unbounded test function $m$, one may need to look at a sequence of bounded approximations of $m$, as discussed in the explanations after Definition~\ref{def:2-smooth-test-functions}. Such an approximation will require the uniform bounds on the derivatives to grow at a specific rate, which in turn affects the constant $C_m$. This is not the focus of this paper, and we leave it to future work.

We now show that $\mathbf{T}_1^{f_*}$ offers an improvement over the classical Bayes-optimal denoiser that only achieves $\Theta(\sigma^2)$ accuracy, as an immediate corollary. 

\begin{corollary}[Moment Matching: Bayes-optimal Denoiser]
	\label{cor:bayes-optimal}
	Consider the same setting and assumptions as in Theorem~\ref{thm:first-order}. The Bayes-optimal denoiser $$\mathbf{T}^\star(y) = y + \sigma^2 \nabla \log q(y)\;,$$ satisfies for $\forall m \in \mathcal{M}^2(\mathbb{R}^d)$,
	\begin{align*}
	  \E_{Y \sim P_Y} [ m(\mathbf{T}^\star(Y)) ] - \E_{X \sim P_X} [ m(X) ]  = \frac{\sigma^2}{2} \E_{Y \sim P_Y} [\langle \nabla m(Y), \nabla \log q(Y) \rangle]  + O(\sigma^4) \;.
	\end{align*}
\end{corollary}
This corollary states a lower bound on the behavior of the Bayes-optimal denoiser in terms of moment matching. In general, $\E_{Y \sim P_Y} [\langle \nabla m(Y), \nabla \log q(Y) \rangle]$ is non-zero, therefore the Bayes-optimal denoiser only achieves $\Theta(\sigma^2)$ accuracy in moment matching. This is in stark contrast to the first-order denoiser $\mathbf{T}_1^{f_*}$ that achieves $O(\sigma^4)$ accuracy, an order-of-magnitude improvement.

When both the signal and noise are Gaussian distributions, one can take $m = -\log q$, a quadratic function. Symbolically, the leading term on the right-hand side of Corollary~\ref{cor:bayes-optimal} is negative, confirming the intuition of over-shrinkage in second moments by the Bayes-optimal denoiser.

\subsection{Second-Order Denoiser}
We now define the class of third-order smooth test functions $\mathcal{M}^3(\mathbb{R}^d)$, which will be used to analyze the performance of the second-order denoiser $\mathbf{T}_2^{f_*, g_*}$.
\begin{definition}[3-Smooth Test Functions]
	\label{def:3-smooth-test-functions}
	The class of 3-smooth test functions $\mathcal{M}^3(\mathbb{R}^d)$ is defined as
	\begin{equation*}
	\cM^{3}(\mathbb{R}^d) := \left\{ m: \mathbb{R}^d \rightarrow \mathbb{R} \;\middle|\;
	\begin{aligned}
	&\sup_{y \in \mathbb{R}^d} \| \nabla^3 m(y) \|_{\infty} < \infty, \quad m, \nabla m ~\text{vanish at infinity} \;, \\
	& \int_{\mathbb{R}^d} |m(y)| \dd y < \infty, \quad \int_{\mathbb{R}^d} \|\nabla m(y) \|_{\infty} \dd y < \infty \;. 
	\end{aligned} 
	\right\} \;.
	\end{equation*}
\end{definition}

\begin{theorem}[Moment Matching: Optimal Second-Order Denoiser]
	\label{thm:second-order}
	Under Assumption~\ref{assump:integrability}(ii), we consider the second-order denoiser $\mathbf{T}_2^{f_*, g_*}$ defined in \eqref{eq:denoiser-second-order}, 
	$$\mathbf{T}_2^{f_*, g_*}(y) = y + \frac{\sigma^2}{2} \nabla \log q(y) - \frac{\sigma^4}{8} \nabla \left( \frac{1}{2} \| \nabla \log q(y) \|^2 + \nabla \cdot \nabla \log q(y) \right)\;.$$
	Let $p = P_X$ satisfy Assumption~\ref{assump:density-smoothness} for up to $k=6$, and the noise $Z$ has bounded 6-th moments as in Assumption~\ref{assump:noise-moments}(i) \& (ii). We then have for $\forall m \in \mathcal{M}^3(\mathbb{R}^d)$, the $3$-smooth test function as in Definition~\ref{def:3-smooth-test-functions},
	\begin{align*}
	 \big| \E_{Y \sim P_Y} [ m(\mathbf{T}_2^{f_*, g_*}(Y)) ] - \E_{X \sim P_X} [ m(X) ] \big| \leq C_{m} \cdot \sigma^6 \;.
	\end{align*}
	Here, the universal constant $C_{m} > 0$ only depends on the test function $m$.
\end{theorem}

This theorem requires slightly stronger conditions on the signal and noise distributions compared to Theorem~\ref{thm:first-order}. In particular, the signal distribution $P_X$ needs to be smooth up to 6-th order, and the noise distribution $P_Z$ needs to have bounded 6-th moments and a Gaussian-like $4$-th moment, namely $\E[Z_i^4] = 3$. Extensions to generic $\E[Z_i^4] = \kappa \in [1, \infty)$ for noise distribution with i.i.d. coordinates are possible, at the price of a slightly more complex second-order denoising equation. The result holds without requiring Gaussian assumptions, merely the knowledge of the 4-th moments of the noise. Under these conditions, the second-order denoiser $\mathbf{T}_2^{f_*, g_*}$ matches the generalized moments (induced by $m$) of the signal distribution $P_X$ up to an error of order $O(\sigma^6)$. This is an order-of-magnitude improvement over the first-order denoiser $\mathbf{T}_1^{f_*}$ that only achieves $O(\sigma^4)$ accuracy. 

Remark that the set of test functions $\mathcal{M}^3(\mathbb{R}^d)  \subseteq \mathcal{M}^2(\mathbb{R}^d)$, as we require bounded third-order derivatives. The trade-off is that with slightly stronger conditions on the signal distribution, noise distribution, and test functions, we can achieve higher-order accuracy in generalized moment matching, with a universal denoising map $\mathbf{T}_2^{f_*, g_*}$ that only depends on the score $\nabla \log q$ and its derivatives.

\section{Optimal Denoising and Monge-Ampère Equations}
\label{sec:monge-ampere-optimality}

In this section, we study how the denoised distribution $\mathbf{T} \sharp P_Y$ matches the signal distribution $P_X$ from the perspective of optimal transport maps and Monge-Ampère equations. We investigate how well the denoisers $\mathbf{T}_1^{f_*}$ and $\mathbf{T}_2^{f_*, g_*}$ approximate the static Monge-Ampère equation that characterizes the optimal transport map from $P_Y$ to $P_X$.

We measure $\mathbf{T} \sharp P_Y$ matches $P_X$ in terms of an affinity between densities motivated by Monge-Ampère equations \cite{caffarelli1992regularity,villani2008optimal,deb2025no}, restricted to a compact domain $\Omega \subseteq \mathbb{R}^d$,
\begin{align}
d_{\mathrm{MA}}( \mathbf{T} \sharp P_Y, P_X ) := \sup_{y \in \Omega} \big| p_Y(y) - p_X(\mathbf{T}(y)) \det( \nabla \mathbf{T}(y) ) \big|  \;.
\end{align}
This affinity stems from the study of static Monge-Ampère equations for optimal transport maps. It directly compares the point-wise match between the densities. Suppose $\mathbf{T}$ is the optimal transport map from $P_Y$ to $P_X$, then the static Monge-Ampère equation \cite{caffarelli1992regularity} states that
\begin{align*}
p_Y(y) = p_X(\mathbf{T}(y)) \det( \nabla \mathbf{T}(y) ) \;, \quad \forall y \in \Omega \;.
\end{align*}
Therefore, $d_{\mathrm{MA}}( \mathbf{T} \sharp P_Y, P_X )$ quantifies how well the Monge-Ampère equation is satisfied for a denoiser map $\mathbf{T}$, uniformly over the domain $\Omega$.

In the case when $\mathbf{T}: \Omega \rightarrow \Omega_X$ is a diffeomorphism on the support $\Omega \subseteq \mathbb{R}^d$ with $\inf_{y \in \Omega}\det( \nabla \mathbf{T}(y) ) > 0$, then $d_{\mathrm{MA}}( \mathbf{T} \sharp P_Y, P_X ) = 0$ if and only if $d_{\mathrm{TV}}( \mathbf{T} \sharp P_Y, P_X) = 0$. In fact, the MA affinity directly quantifies how the densities match point-wise $|p_{\mathbf{T}(Y)}(x) - p_X(x)|$, after accounting for the volume change induced by the transport map $x = \mathbf{T}(y)$.

One can verify that the optimal denoisers $\mathbf{T}_1^{f_*}$ and $\mathbf{T}_2^{f_*, g_*}$ are diffeomorphisms when $\eta$ is sufficiently small, as they are the gradient of strongly convex functions.
Therefore, for theoretical investigations conducted in this section, we require a small-noise condition on $\eta$.
\begin{assumption}[Small Noise Level]
	\label{assump:small-noise}
	Assume that the noise parameter $\eta = \sigma^2/2  \in [0, \delta)$ with a sufficiently small positive constant $\delta$ such that $\mathbf{T}_1^{f_*}$ and $\mathbf{T}_2^{f_*,g_*}$ are gradient of convex functions on $\Omega$, namely: 
	
	(i) $\inf_{y \in \Omega} I_d + \eta \nabla^2 f_* (y)  \succ 0$; 
	
	(ii) $\inf_{y \in \Omega} I_d + \eta \nabla^2 f_* (y) + \tfrac{\eta^2}{2} \nabla^2 g_* (y) \succ 0$.
\end{assumption}

As we evaluate the densities directly in this section, we require the following boundedness assumption on the denoisers induced by $f_*, g_*$ over $\Omega$, the domain of interest. This regularity condition is stronger than Assumption~\ref{assump:integrability} and requires point-wise boundedness. For simplicity, we only concern a compact domain $\Omega \subseteq \mathbb{R}^d$, which may be taken as a subset of the support of $P_Y$. 

\begin{assumption}[Boundedness of Denoisers]
	\label{assump:bdd-hessian}
	Recall the definition of $f_*, g_*$ in \eqref{eq:denoiser-first-order} and \eqref{eq:denoiser-second-order} that depends on $q = P_Y$.  The density $q$ satisfies, for any $\eta \in [0, c)$ for some constant $c > 0$:
	
	(i) $\sup_{y \in \Omega}  \| \nabla f_*(y) \| + \| \nabla^2 f_*(y) \| < \infty$;

	(ii) $\sup_{y \in \Omega} \| \nabla g_*(y) \| + \| \nabla^2 g_*(y) \| < \infty$.
\end{assumption}

Finally, we remark that a compact domain restriction is without loss of generality even in the case when $P_Y, P_X$ are supported on $\mathbb{R}^d$. One may define a domain $\Omega(\varepsilon)$ such that $p_Y(y), p_X(y) < \varepsilon/2, \forall y \notin \Omega(\varepsilon)$ for an arbitrarily small $\varepsilon > 0$, and considers the truncated denoiser
$$ \tilde{\mathbf{T}}(y) = \begin{cases} \mathbf{T}(y), & y \in \Omega(\varepsilon) \\ y, & y \notin \Omega(\varepsilon) \end{cases} \;.$$
Then, the MA affinity over the unbounded region satisfies
\begin{align*}
 \sup_{y \in \mathbb{R}^d} \big| p_Y(y) - p_X(\tilde{\mathbf{T}}(y)) \det( \nabla \tilde{\mathbf{T}}(y) ) \big|  \leq d_{\mathrm{MA}}( \mathbf{T} \sharp P_Y, P_X ) + \varepsilon \;.
\end{align*}
For any target accuracy level $\varepsilon  = \sigma^{k}, k \in \mathbb{N}^+$, one can work on the compact domain $\Omega(\varepsilon)$.

Now we are ready to state the main results in terms of distributional matching via the Monge-Ampère equation, where the differential equations for optimal denoisers given in Section~\ref{sec:denoising-ode} will arise naturally.

\subsection{First-Order Denoising Equation}

\begin{theorem}[Distributional Matching: Optimal First-Order Denoiser]
	\label{thm:distributional-matching-first-order}
	Under Assumptions~\ref{assump:small-noise}(i) and \ref{assump:bdd-hessian}(i), we consider the first-order denoiser $\mathbf{T}_1^{f}(y)$.

	Let $p = P_X$ satisfy Assumption~\ref{assump:density-smoothness} for up to $k=4$, and the noise $Z$ has bounded 4-th moments as in Assumption~\ref{assump:noise-moments}(i). Then with some universal constant $C_{\Omega} > 0$,
	\begin{align*}
	d_{\mathrm{MA}}( \mathbf{T}_1^{f} \sharp P_Y, P_X ) \leq C_{\Omega} \cdot \sigma^4 \;,
	\end{align*}
	if $f$ satisfies the differential equation
	$$
	q \nabla f =  \nabla q \;, 
	$$
\end{theorem}

This theorem states that under mild conditions on the signal and noise distributions, the first-order denoiser $\mathbf{T}_1^{f}$ matches the signal distribution $P_X$ in terms of the Monge-Ampère affinity up to an error of order $O(\sigma^4)$, provided that $f$ satisfies the given differential equation. This differential equation justifies the optimality of the first-order denoiser $\mathbf{T}_1^{f_*}$, defined in \eqref{eq:denoiser-first-order}, as it approximates the Monge-Ampère equation up to $O(\sigma^4)$ accuracy. We call this universal denoiser because its validity in terms of distributional matching holds as long as the signal distribution $P_X$ is smooth enough (up to 4-th order), and the noise distribution $P_Z$ has bounded 4-th moments and is uncorrelated.

An immediate corollary is a lower bound showing that the classical Bayes-optimal denoiser only achieves $\Theta(\sigma^2)$ accuracy in terms of distributional matching, thus is suboptimal compared to the first-order denoiser $\mathbf{T}_1^{f_*}$, in the small constant noise regime.

\begin{corollary}[Distributional Matching: Bayes-optimal Denoiser]
	\label{cor:distributional-matching-bayes-optimal}
	Consider the same setting and assumptions as in Theorem~\ref{thm:distributional-matching-first-order}. The Bayes-optimal denoiser $$\mathbf{T}^\star(y) = y + \sigma^2 \nabla \log q(y)\;,$$ satisfies
	\begin{align*}
	  d_{\mathrm{MA}}( \mathbf{T}^\star \sharp P_Y, P_X ) \geq  \frac{\sigma^2}{2}  \sup_{y \in \Omega} \Delta q (y) - C_{\Omega} \cdot \sigma^4 \;,
	\end{align*}
	for some universal constant $C_{\Omega} > 0$.
\end{corollary}

\subsection{Second-Order Denoising Equation}

The derivation for the second-order denoising equation is more involved. The following theorem shows the optimality of the second-order denoiser $\mathbf{T}_2^{f, g}$ defined in \eqref{eq:denoiser-second-order}, and derives the corresponding differential equations that $f, g$ need to satisfy.

\begin{theorem}[Distributional Matching: Optimal Second-Order Denoiser]
	\label{thm:distributional-matching-second-order}
	Under Assumptions~\ref{assump:small-noise}(i) \& (ii) and \ref{assump:bdd-hessian} (i) \& (ii), we consider the second-order denoiser $\mathbf{T}_2^{f, g}$.

	Let $p = P_X$ satisfy Assumption~\ref{assump:density-smoothness} for up to $k=6$, and the noise $Z$ has bounded 6-th moments as in Assumption~\ref{assump:noise-moments}(i) \& (ii). Then with some universal constant $C_{\Omega} > 0$,
	\begin{align*}
	d_{\mathrm{MA}}( \mathbf{T}_2^{f, g} \sharp P_Y, P_X ) = C_{\Omega} \cdot \sigma^6 \;,
	\end{align*}
	if $f, g$ satisfy the differential equations
	\begin{align}
	q \nabla f &=  \nabla q \;, \nonumber  \\
	q \nabla g &= \nabla^2 f \nabla q + \Delta q \nabla f - \nabla \Delta q \;. \label{eqn:second-order-DE}
	\end{align}
\end{theorem}

Again, this theorem requires stronger conditions than Theorem~\ref{thm:distributional-matching-first-order}, in terms of smoothness of the signal distribution $P_X$ (up to 6-th order), and boundedness of 6-th moments and a Gaussian-like 4-th moment assumption, stated in Assumption~\ref{assump:noise-moments}(ii). If instead one wishes to consider general noise $Z$ with i.i.d. coordinates that have general 4-th moments $\E[Z_i^4] = \kappa \in [1, \infty)$, then the second-order denoising equation will be modified by adding a term of the form $(3 - \kappa) \partial_i^3 q$ for each coordinate $i$ in \eqref{eqn:second-order-DE}. We choose to state a less general version for simplicity of presentation.

\section{Numerics}
\label{sec:numerics}

In this section, we implement the score matching techniques \cite{hyvarinen2005estimation} as in \eqref{eqn:fisher-divergence-minimization} to learn the score function from samples of the noisy distribution $P_Y$, and then use it to obtain the estimated first-order and second-order denoisers as in \eqref{eqn:est-first-order-denoiser} and \eqref{eqn:est-second-order-denoiser}. We compare the denoising quality of these two denoisers on synthetic 2D datasets, including four types of signal distributions $P_X$: (i) correlated Gaussian distribution, (ii) mixture of Gaussians, (iii) uniform distribution on a square, and (iv) infinite mixture of Gaussians centered on a circle. 

In each experiment, we parametrize the score function as a neural network $\xi: \mathbb{R}^2 \rightarrow \mathbb{R}^2$, with three hidden layers each of width 64. For each hidden layer, we use a residual connection and tanh activation (for smoothness), namely $z \rightarrow z + \tanh(W z + b), z \in \mathbb{R}^{64}$ with $W \in \mathbb{R}^{64 \times 64}, b \in \mathbb{R}^{64}$ being the weights and biases.  The noise $Z$ is taken as standard Gaussian noise. We set the noise parameter $\eta = \sigma^2/2 = 0.5$. We implement the score matching objective \eqref{eqn:fisher-divergence-minimization} using PyTorch \cite{paszke2019pytorch}, and optimize the parameters of the neural network using the Adam optimizer with a learning rate of $0.001$ for $10$ epochs on a dataset of $n=6400$ noisy samples $Y$. After obtaining the estimated score function $\widehat{\xi}$, we plug in to obtain the estimated first-order and second-order denoisers as in \eqref{eqn:est-first-order-denoiser} and \eqref{eqn:est-second-order-denoiser}. We generate the $m=1000$ test samples from the noisy distribution $P_Y$, and their denoised versions using the estimated first-order and second-order denoisers in Figures~\ref{fig:simu-2d-gauss-scorenn} and \ref{fig:simu-2d-square-torus-scorenn}, and contrast with the ground-truth $m=1000$ test samples from $P_X$. For each experiment, we compute the average Wasserstein distance between the $m=1000$ denoised samples and the ground-truth samples and report it in each subplot.

\begin{figure}[!ht]
	\centering
	\includegraphics[width=0.88\textwidth]{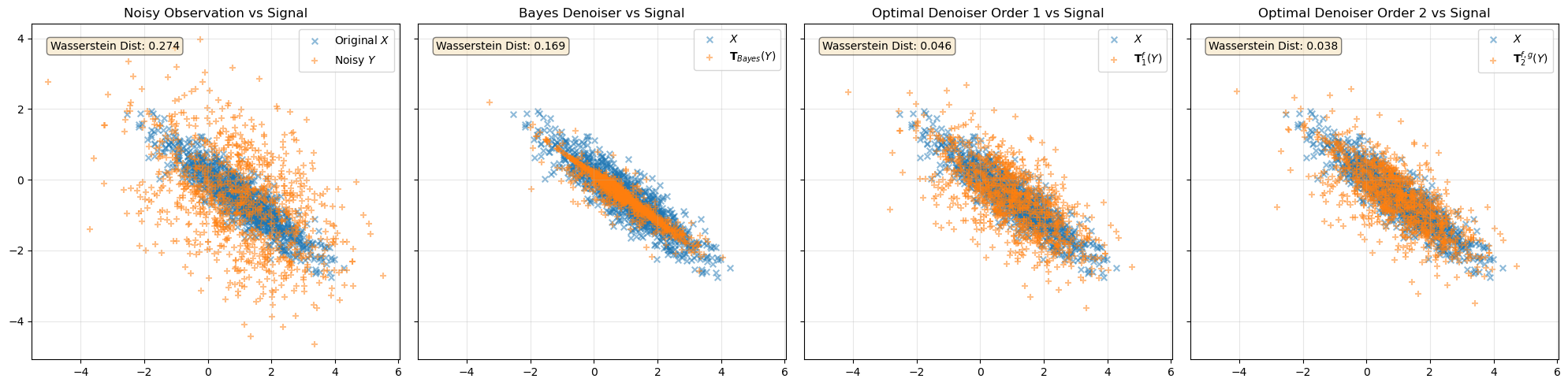}
	\includegraphics[width=0.88\textwidth]{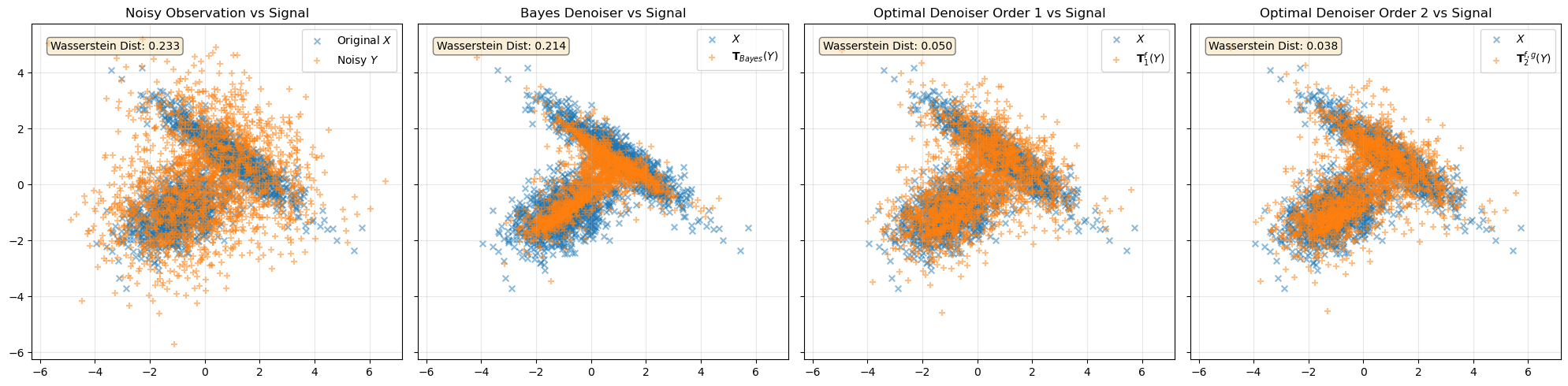}
	\caption{\scriptsize Comparison of denoising quality using score matching. We compare four denoisers: no-shrinkage $\mathbf{T}^{\mathrm{id}}$, Bayes-optimal shrinkage $\mathbf{T}^\star$, our proposed first-order denoiser $\mathbf{T}_1^{f_*}$, and second-order denoiser $\mathbf{T}_2^{f_*, g_*}$. The Wasserstein distance $W(\mathbf{T} \sharp P_Y, P_X)$ is reported in each subplot. Top row: correlated Gaussian distribution for $P_X$, with $W(\mathbf{T}^{\mathrm{id}} \sharp P_Y, P_X) = 0.274$, $W(\mathbf{T}^\star \sharp P_Y, P_X) = 0.169$, $W(\mathbf{T}_1^{f_*} \sharp P_Y, P_X) = 0.046$, $W(\mathbf{T}_2^{f_*, g_*} \sharp P_Y, P_X) = 0.038$; Bottom row: mixture of two correlated Gaussians for $P_X$, with $W(\mathbf{T}^{\mathrm{id}} \sharp P_Y, P_X) = 0.233$, $W(\mathbf{T}^\star \sharp P_Y, P_X) = 0.214$, $W(\mathbf{T}_1^{f_*} \sharp P_Y, P_X) = 0.050$, $W(\mathbf{T}_2^{f_*, g_*} \sharp P_Y, P_X) = 0.038$. Visually, both first-order and second-order denoisers produce significantly better distributional matching than the Bayes-optimal denoiser. Bayes-optimal denoisers introduce overly aggressive shrinkage, whereas our proposed optimal denoisers introduce a more balanced shrinkage. Quantitatively, the Bayes-optimal denoiser only slightly improves over no shrinkage, while both first-order and second-order denoisers achieve an order of magnitude improvement in Wasserstein distance. Furthermore, the second-order denoiser outperforms the first-order denoiser by a margin, demonstrating the potential benefit of higher-order denoising.}
	\label{fig:simu-2d-gauss-scorenn}
\end{figure}

Figure~\ref{fig:simu-2d-gauss-scorenn} presents two experiments.
On the top row, we specify the signal distribution $P_X \sim \cN(\mu_1, \Sigma_1)$ where $\mu_1 = \begin{bmatrix}1.0 \\ -0.5 \end{bmatrix}$ and $\Sigma_1 = \begin{bmatrix}1.5 & -1.0 \\ -1.0 & 0.8 \end{bmatrix}$, the same setting as Figure~\ref{fig:intro-2d} where we denoise using the analytic score function $\nabla \log q (y) = -\Sigma_1^{-1}(y - \mu_1)$. Here, we estimate using score matching. This first example demonstrates that score matching aligns with the analytic score function in terms of denoising quality, with similar qualitative and quantitative performance.
 On the bottom row, we specify the signal distribution $P_X$ as a mixture of Gaussians with two components, with an additional component $\cN(\mu_2, \Sigma_2)$ where $\mu_2 = \begin{bmatrix}-1.0 \\ -1.0 \end{bmatrix}$ and $\Sigma_2 = \begin{bmatrix}0.8 & 0.3 \\ 0.3 & 0.5 \end{bmatrix}$ besides the component $\cN(\mu_1, \Sigma_1)$, each with equal weight. 
 
 In both experiments, we observe the following phenomenon. First, the Bayes-optimal denoiser introduces an overly aggressive shrinkage, thus shrinking the distribution $\mathbf{T}^\star \sharp P_Y$ inside the mass of $P_X$ and producing poor distributional matching. In the mixture of Gaussian experiment, the $W(\mathbf{T}^\star \sharp P_Y, P_X) = 0.214$, not a good improvement compared to no shrinkage $W(P_Y, P_X) = 0.233$. In contrast, both the first-order and second-order denoisers produce significantly better distributional matching, with $W(\mathbf{T}_1^{f_*} \sharp P_Y, P_X) = 0.050$ and $W(\mathbf{T}_2^{f_*, g_*} \sharp P_Y, P_X) = 0.038$, an order of magnitude improvement. Second, although visually the first-order and second-order denoisers look similar, the second-order denoiser achieves a better Wasserstein distance in both experiments by a margin, demonstrating the benefit of higher-order denoising. Both of our proposed denoisers significantly outperform the Bayes-optimal denoiser in terms of distributional matching, both visually and quantitatively. This is compatible with our theoretical findings in Theorems~\ref{thm:first-order}-\ref{thm:distributional-matching-second-order}.

\begin{figure}[!ht]
	\centering
	\includegraphics[width=0.88\textwidth]{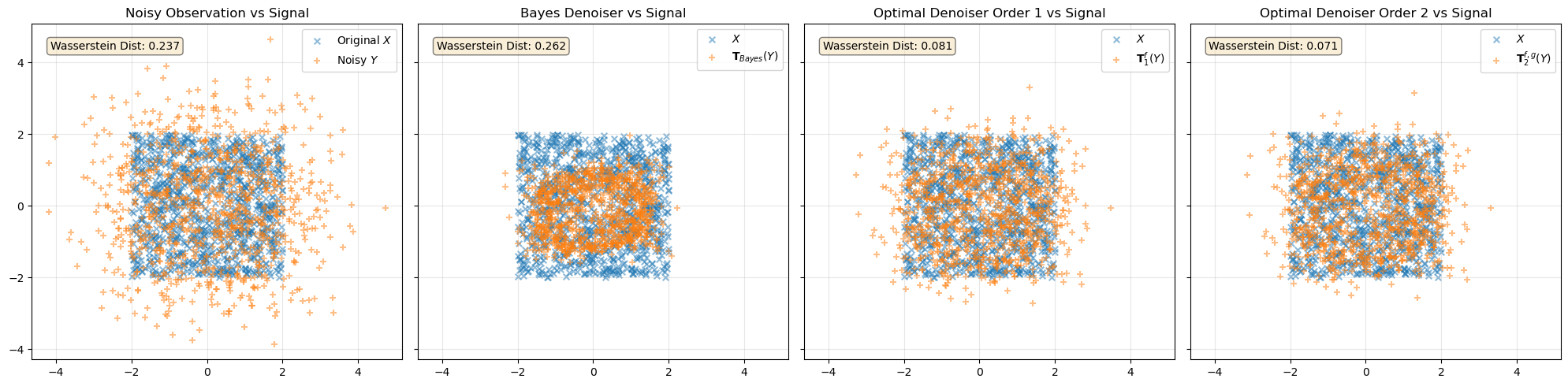}
	\includegraphics[width=0.88\textwidth]{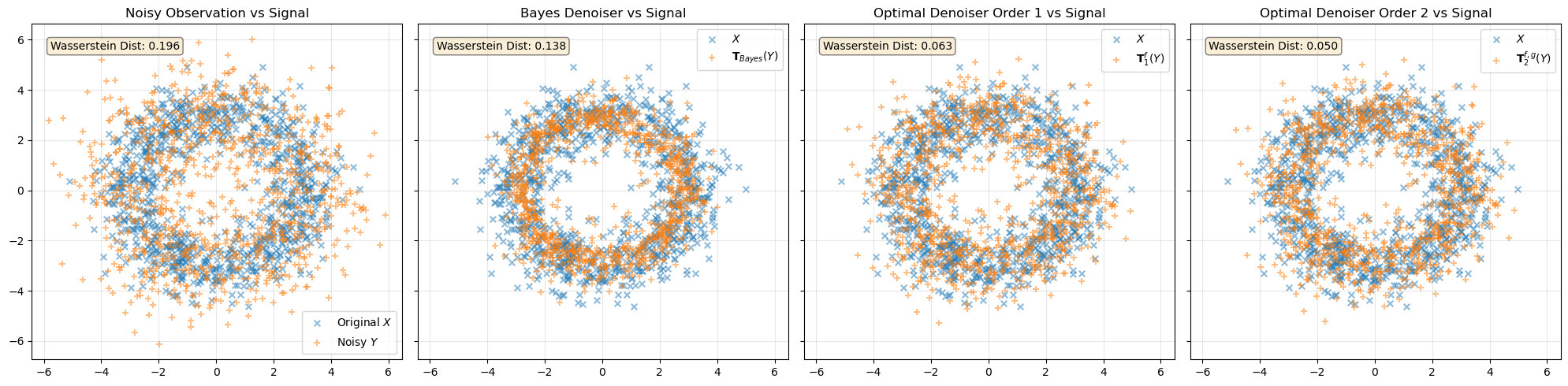}
	\caption{\scriptsize Comparison of denoising quality using score matching. We compare four denoisers as in Figure~\ref{fig:simu-2d-gauss-scorenn}. Top row: $P_X$ on a square, namely a Uniform distribution $[-2, 2]^2$, with $W(\mathbf{T}^{\mathrm{id}} \sharp P_Y, P_X) = 0.237$, $W(\mathbf{T}^\star \sharp P_Y, P_X) = 0.262$, $W(\mathbf{T}_1^{f_*} \sharp P_Y, P_X) = 0.081$, $W(\mathbf{T}_2^{f_*, g_*} \sharp P_Y, P_X) = 0.071$; Bottom row: $P_X$ on a torus, namely an infinite mixture of Gaussians with mean uniformly supported on a circle $r=3$ and variance $0.5$, with $W(\mathbf{T}^{\mathrm{id}} \sharp P_Y, P_X) = 0.196$, $W(\mathbf{T}^\star \sharp P_Y, P_X) = 0.138$, $W(\mathbf{T}_1^{f_*} \sharp P_Y, P_X) = 0.063$, $W(\mathbf{T}_2^{f_*, g_*} \sharp P_Y, P_X) = 0.050$. }
	\label{fig:simu-2d-square-torus-scorenn}
\end{figure}

Figure~\ref{fig:simu-2d-square-torus-scorenn} shows two more complex distributions for $P_X$. On the top row, we define $P_X$ as a uniform distribution over the square $[-2, 2]^2$. On the bottom row, $P_X$ is an infinite mixture of Gaussians with means uniformly supported on a circle of radius $r=3$ and variance $0.5$. The over-shrinkage effect is evident in the Bayes-optimal denoiser in these two experiments, leading to poor reconstruction of the signal distribution. In both the square and torus experiments, the Bayes-optimal denoiser misses some non-trivial support of $P_X$, as indicated by the large Wasserstein distances. For the square experiment, $W(\mathbf{T}^\star \sharp P_Y, P_X) = 0.262$ performs worse than no-shrinkage $W(P_Y, P_X) = 0.237$. For the torus experiment, $W(\mathbf{T}^\star \sharp P_Y, P_X) = 0.138$ only improves modestly over no-shrinkage $W(P_Y, P_X) = 0.196$. In these two experiments, the first- and second-order denoisers again outperform the Bayes-optimal denoiser, both visually and quantitatively. We note that, with these complex distributions, the performance difference between first-order and second-order denoisers is less pronounced than in earlier experiments. For the square experiment, the improvement is from $W(\mathbf{T}_1^{f_*} \sharp P_Y, P_X) = 0.081$ to $W(\mathbf{T}_2^{f_*, g_*} \sharp P_Y, P_X) = 0.071$. Similarly, for the torus experiment, the improvement goes from $W(\mathbf{T}_1^{f_*} \sharp P_Y, P_X) = 0.063$ to $W(\mathbf{T}_2^{f_*, g_*} \sharp P_Y, P_X) = 0.050$. This could be due to (i) the limited expressiveness of the neural network used to estimate the score function, which may not accurately capture higher-order derivatives of the score function needed for second-order denoising, and (ii) the numerical instability of higher-order derivatives of the score function. Nonetheless, both first- and second-order denoisers significantly outperform the Bayes-optimal denoiser in distributional matching, consistent with our theoretical results.

\section*{Acknowledgements}
The author acknowledges the support from the NSF Career Award (DMS-2042473) and the Wallman Society of Fellows at the University of Chicago. The author also thanks Nabarun Deb, Shuo-Chieh Huang, Nikos Ignatiadis, Ben Recht, and Bodhi Sen for comments on an earlier draft. This work is part I of a two-part series on distributional shrinkage; part II can be found in \cite{liang2025distributionalshrinkageiioptimal}.

\ifbiblatex
  \printbibliography
\fi

\appendix

\section{Proofs of Theorems}
\label{sec:proofs-of-theorems}
\subsection{Proofs for Section~\ref{sec:moment-matching}}

\begin{proof}[Proof of Theorem~\ref{thm:first-order}]
	Define $\eta$ such that $\eta := \sigma^2/2$ to streamline the calculations.
	By definition of the push-forward map $\mathbf{T}_1^{f}$, we have
	\begin{align}
	\E_{Y \sim P_Y} [ m(\mathbf{T}_1^{f}(Y)) ] - \E_{X \sim P_X} [ m(X) ] &= \int_{\mathbb{R}^d} m(y+\eta \nabla f(y)) q(y) \dd{y} - \int_{\mathbb{R}^d} m(x) p(x) \dd{x} \;. \label{eqn:moment-matching-first-order}
	\end{align}
	Recall Taylor expansion of $m(y+\eta \nabla f(y))$ along the line segment interpolating $y$ and $y+\eta \nabla f(y)$ with mean-value form of the remainder, we know that, there exists some $\delta = \delta(y, \eta, m, f) \in [0,1]$, such that
	\begin{align*}
	m(y+\eta \nabla f(y)) &= m(y) + \eta \langle \nabla m(y), \nabla f(y) \rangle + \frac{\eta^2}{2} \langle \nabla^2 m(y+\delta \eta \nabla f(y)), \nabla f(y) \otimes \nabla f(y) \rangle \;.
	\end{align*}
	Plugging the above expansion into the integral,
	\begin{align}
	Eqn.~\eqref{eqn:moment-matching-first-order} &= \int_{\mathbb{R}^d} m(y) q(y) \dd{y} -  	\int_{\mathbb{R}^d} m(y) p(y) \dd{y} \label{eqn:moment-difference-pq} \\
	&\quad + \eta \int_{\mathbb{R}^d} \langle \nabla m(y), \nabla f(y) \rangle q(y) \dd{y} \\
	&\quad  + \frac{\eta^2}{2} \int_{\mathbb{R}^d} \langle \nabla^2 m(y+\delta \eta \nabla f(y)), \nabla f(y) \otimes \nabla f(y) \rangle q(y) \dd{y}  \;. \label{eqn:moment-difference-eta2}
	\end{align}
	We analyze each term separately. First, we analyze the term in \eqref{eqn:moment-difference-eta2}. By the Cauchy-Schwarz inequality,
	\begin{align*}
	\int_{\mathbb{R}^d} |\partial_i f(y) \partial_j f(y)| q(y) \dd{y} \leq \left( \int_{\mathbb{R}^d} |\partial_i f(y)|^2 q(y) \dd{y} \right)^{1/2} \left( \int_{\mathbb{R}^d} |\partial_j f(y)|^2 q(y) \dd{y} \right)^{1/2}
	\end{align*}
	if (i) $\int_{\mathbb{R}^d}  \| \nabla f(y)\|^2 q(y) \dd{y} < \infty$ and (ii) $\sup_{y \in \mathbb{R}^d}~\| \nabla^2 m (y) \|_{\infty} < \infty$, we know the term in \eqref{eqn:moment-difference-eta2} is bounded by $O(\eta^2)$. The constant here is independent of $y$, only depending on $m$.

	Moving to the first term on the right-hand side of \eqref{eqn:moment-difference-pq},
	by Lemma~\ref{lem:convolution-expansion}(i), we have
	\begin{align*}
	q(y) = p(y) + \eta \Delta p(y) + O(\eta^2) \;,\\
	\partial_i q(y) = \partial_i p(y) + O(\eta) \;.
	\end{align*}
	Here, the big-O terms are uniform in $y$.
	
	Now 
	\begin{align*}
	\int_{\mathbb{R}^d} m(y) q(y) \dd{y} - \int_{\mathbb{R}^d} m(y) p(y) \dd{y} &= \eta \int_{\mathbb{R}^d} m(y) \Delta p(y) \dd{y} + O(\eta^2) \;, ~\text{by (i)}\\
	& = - \eta \int_{\mathbb{R}^d} \langle \nabla m(y), \nabla p(y) \rangle \dd{y} + O(\eta^2) \;, ~\text{by (ii)} \\
	&= - \eta \int_{\mathbb{R}^d} \langle \nabla m(y), \nabla q(y) \rangle \dd{y} + O(\eta^2) \;, ~\text{by (i)} 
	\end{align*}
	if (i) $\int_{\mathbb{R}^d} |m(y)| \dd y < \infty$, $\int_{\mathbb{R}^d} |\partial_i m(y)| \dd y < \infty$, and (ii) $m \nabla p$ vanishes as $y \rightarrow \infty$ because $m$ vanishes and $\nabla p$ is bounded due to Assumption~\ref{assump:density-smoothness}. To be more specific, (ii) is a direct application of integration by parts
	\begin{align*}
	\nabla \cdot (m \nabla p) = \langle \nabla m, \nabla p \rangle + m \Delta p \;,
	\end{align*}
	and thus $\int_{\mathbb{R}^d} m \Delta p \dd y = - \int_{\mathbb{R}^d} \langle \nabla m, \nabla p \rangle \dd y$.

	Putting things together, we have
	\begin{align*}
	Eqn.~\eqref{eqn:moment-matching-first-order} &= \eta \left[ - \int_{\mathbb{R}^d} \langle \nabla m(y), \nabla q(y) \rangle \dd{y}  + \int_{\mathbb{R}^d} \langle \nabla m(y), \nabla f(y) \rangle q(y) \dd{y}   \right] + O(\eta^2)
	\end{align*}
	and thus
	\begin{align*}
	 \E_{Y \sim P_Y} [ m(\mathbf{T}_1^{f}(Y)) ] - \E_{X \sim P_X} [ m(X) ] &= \eta \int_{\mathbb{R}^d} \langle \nabla m(y), q(y) \nabla f(y) - \nabla q(y) \rangle \dd{y} + O(\eta^2) \;.
	\end{align*}
	Choosing $f = f_* = \log q$ such that $q(y) \nabla f_*(y) = \nabla q(y)$, we finish the proof.
\end{proof}

\begin{proof}[Proof of Corollary~\ref{cor:bayes-optimal}]
	Recall the last equation in the proof of Theorem~\ref{thm:first-order}, 
	\begin{align*}
	 \E_{Y \sim P_Y} [ m(\mathbf{T}_1^{f}(Y)) ] - \E_{X \sim P_X} [ m(X) ] &= \eta \int_{\mathbb{R}^d} \langle \nabla m(y), q(y) \nabla f(y) - \nabla q(y) \rangle \dd{y} + O(\eta^2) \;.
	\end{align*}
	Choosing $f = 2 \log q$, we have
	$$\mathbf{T}_1^{f}(y) = y + \sigma^2 \nabla \log q(y)\;,$$
	which is the Bayes-optimal denoiser. Substituting this into the above equation yields the desired result.
\end{proof}

\begin{proof}[Proof of Theorem~\ref{thm:second-order}]
	Let $\eta := \sigma^2/2$ to streamline the calculations. For simplicity, we use $f = f_*$ and $g = g_*$, defined in \eqref{eq:denoiser-second-order}, in the following proof whenever there is no ambiguity.

	By definition of the push-forward map $\mathbf{T}_2^{f,g}$, we have
	\begin{align*}
	\E_{Y \sim P_Y} [ m(\mathbf{T}_2^{f,g}(Y)) ] - \E_{X \sim P_X} [ m(X) ] &= \int_{\mathbb{R}^d} m\left(y+\eta \nabla f(y) + \frac{\eta^2}{2} \nabla g(y)\right) q(y) \dd{y} - \int_{\mathbb{R}^d} m(x) p(x) \dd{x} \;. 
	\end{align*}
	Recall the Taylor expansion of $m\left(y+\eta \nabla h\right)$ at $y$ up to the third-order, with $h := f + \frac{\eta}{2} g$, we know
	\begin{align*}
	m\left(y+\eta \nabla h(y)\right) &= m(y) + \eta \langle \nabla m(y), \nabla h(y) \rangle + \frac{\eta^2}{2} \langle \nabla^2 m(y), \nabla h(y) \otimes \nabla h(y) \rangle + O(\eta^3) \;.
	\end{align*}
	The constant in the term $O(\eta^3)$ depends on $\sup_{y \in \mathbb{R}^d} \| \nabla^3 m(y) \|_{\infty}$ and $ \| \nabla h(y) \|^3 $, with the first bounded (Assumption~\ref{assump:density-smoothness}) and second integrable (Assumption~\ref{assump:integrability}).

	Rearrange the terms according to the order of $\eta$, we have
	\begin{align*}
		m\left(y+\eta \nabla f(y) + \frac{\eta^2}{2} \nabla g(y)\right) &= m(y) + \eta \langle \nabla m(y), \nabla f(y) \rangle \\
		&\quad + \frac{\eta^2}{2} \left[ \langle \nabla m(y), \nabla g(y) \rangle + \langle \nabla^2 m(y), \nabla f(y) \otimes \nabla f(y) \rangle \right] + O(\eta^3) \;.
	\end{align*}

	By Lemma~\ref{lem:convolution-expansion}(ii), we have
	\begin{align*}
	q(y) = p(y) + \eta \Delta p(y) + \frac{\eta^2}{2} \Delta^2 p(y) + O(\eta^3) \;,
	\end{align*}
	as a result,
	\begin{align*}
	\int_{\mathbb{R}^d} m(y) q(y) \dd{y} - \int_{\mathbb{R}^d} m(y) p(y) \dd{y} &= \eta \int_{\mathbb{R}^d} m(y) \Delta p(y) \dd{y} + \frac{\eta^2}{2} \int_{\mathbb{R}^d} m(y) \Delta^2 p(y) \dd{y} + O(\eta^3) \;,\\
	&= - \eta \int_{\mathbb{R}^d} \langle \nabla m(y), \nabla p(y) \rangle \dd{y} \\
	&\quad - \frac{\eta^2}{2} \int_{\mathbb{R}^d} \langle \nabla m(y), \nabla \Delta p(y) \rangle \dd{y} + O(\eta^3) \;,
	\end{align*}
	where the last line follows from integration by parts
	\begin{align*}
	\nabla \cdot (m \nabla p) = \langle \nabla m, \nabla p \rangle + m \Delta p \;, \\
	\nabla \cdot (m \nabla \Delta p) = \langle \nabla m, \nabla \Delta p \rangle + m \Delta^2 p \;,
	\end{align*}
	and the boundary condition as $m \nabla p, m \nabla \Delta p$ vanishes on the boundary as $y\rightarrow \infty$. Here we use the fact that $m$ vanishes and $\nabla p, \nabla \Delta p$ are bounded (Assumption~\ref{assump:density-smoothness}).
	Now invoke Lemma~\ref{lem:convolution-expansion}(ii) again, we have element-wise the following expansions
	\begin{align}
	\nabla p(y) &=  \nabla q(y) - \eta \nabla \Delta p(y) + O(\eta^2) \;, \label{eqn:q'_eta^2}\\
	\nabla \Delta p(y) &=   \nabla \Delta q(y) + O(\eta) \label{eqn:q'''_eta} \;.
	\end{align}
	Hence, 
	\begin{align*}
		&- \eta \int_{\mathbb{R}^d} \langle \nabla m(y), \nabla p(y) \rangle \dd{y} - \frac{\eta^2}{2} \int_{\mathbb{R}^d} \langle \nabla m(y), \nabla \Delta p(y) \rangle \dd{y} \\
		&= - \eta \int_{\mathbb{R}^d} \langle \nabla m(y), \nabla q(y) \rangle \dd{y} + \frac{\eta^2}{2} \int_{\mathbb{R}^d} \langle \nabla m(y), \nabla \Delta p(y) \rangle \dd{y} + O(\eta^3) \;, ~\text{by \eqref{eqn:q'_eta^2}}\\
	&= - \eta \int_{\mathbb{R}^d} \langle \nabla m(y), \nabla q(y) \rangle \dd{y} + \frac{\eta^2}{2} \int_{\mathbb{R}^d} \langle \nabla m(y), \nabla \Delta q(y) \rangle \dd{y} + O(\eta^3) \;, ~\text{by \eqref{eqn:q'''_eta}}
	\end{align*}
	where the two lines follows from $\int_{\mathbb{R}^d} |\partial_i m(y)| \dd y < \infty$, for all $i$.

	Observe that the following identity holds, by applying the chain-rule recursively,
	\begin{align}
	\nabla \cdot ( q \langle \nabla m, \nabla f \rangle  \nabla f ) &= \langle \nabla (q \langle \nabla m, \nabla f \rangle ), \nabla f \rangle + q \langle \nabla m, \nabla f \rangle \Delta f \nonumber \\
	& = \langle \nabla^2 m, \nabla f \otimes \nabla f \rangle q + \langle \nabla m, \nabla^2 f \nabla f \rangle q + \langle \nabla q, \langle \nabla m, \nabla f \rangle \nabla f \rangle + q \langle \nabla m, \nabla f \rangle \Delta f  \label{eqn:long-chain-rule}\;.
	\end{align}
	Note that $ q \langle \nabla m, \nabla f \rangle  \nabla f$ vanishes on the boundary as $y\rightarrow \infty$ due to the fact that $\partial_i m$ vanishes and $q \partial_i f \partial_j f$ is bounded (Assumption~\ref{assump:integrability}). Therefore, applying integration by parts and recalling \eqref{eqn:long-chain-rule}, we have
	\begin{align*}
	&\int_{\mathbb{R}^d} \langle \nabla^2 m(y), \nabla f(y) \otimes \nabla f(y) \rangle q(y) \dd{y} \\
	&= - \int_{\mathbb{R}^d} \Big[ \langle \nabla m(y), \nabla^2 f(y) \nabla f(y) \rangle q(y) +\langle \nabla m(y), \nabla f(y) \rangle \langle \nabla f(y), \nabla q(y) \rangle \\
	&\qquad\qquad\quad +  \langle \nabla m(y), \Delta f(y) \nabla f(y) \rangle q(y)  \Big] \dd{y}  \;.
	\end{align*}

	Collect all terms with $\frac{\eta^2}{2}$, we have
	\begin{align*}
	&\E_{Y \sim P_Y} [ m(\mathbf{T}_2^{f,g}(Y)) ] - \E_{X \sim P_X} [ m(X) ] \\
	&= \frac{\eta^2}{2} \int_{\mathbb{R}^d} \Big \langle \nabla m, q \nabla g - (\nabla^2 f \nabla f + \Delta f \nabla f) q - \langle \nabla f, \nabla q \rangle \nabla f + \nabla \Delta q \Big \rangle \dd y + O(\eta^3) \;.
	\end{align*}
	Using Lemma~\ref{lem:log-density-identities}, and recall that $f_* = \log q$, we have
	\begin{align*}
	(\nabla^2 f_* \nabla f_* + \Delta f_* \nabla f_*) + \langle \frac{\nabla q}{q}, \nabla f_* \rangle \nabla f_* - \frac{\nabla \Delta q}{q} = -\nabla (\Delta f_* + \frac{1}{2} \| \nabla f_* \|^2 ) \;.
	\end{align*}

	Choosing $f = f_*$ and $g = g_*$ defined in \eqref{eq:denoiser-second-order}, one can verify that
	\begin{align*}
	 \nabla g &=  (\nabla^2 f \nabla f + \Delta f \nabla f) + \langle \frac{\nabla q}{q}, \nabla f \rangle \nabla f - \frac{\nabla \Delta q}{q}   \;,
	\end{align*}
	and thus the $\E_{Y \sim P_Y} [ m(\mathbf{T}_2^{f,g}(Y)) ] - \E_{X \sim P_X} [ m(X) ] = O(\eta^3)$, finishing the proof.
\end{proof}

\subsection{Proofs for Section~\ref{sec:monge-ampere-optimality}}

\begin{proof}[Proof of Theorem~\ref{thm:distributional-matching-first-order}]
	Recall the scaling $\eta := \sigma^2/2$, and set $f = f^\star$ for simplicity of notation. By definition of $d_{\mathrm{MA}}$, and that $p = P_X$, $q = P_Y$, we have
	\begin{align*}
	d_{\mathrm{MA}}( \mathbf{T}_1^{f} \sharp P_Y, P_X ) &= \sup_{y \in \Omega} \big| q(y) - p(y + \eta \nabla f(y)) \det( I_d + \eta \nabla^2 f(y) ) \big| \;.
	\end{align*}

	Recall Lemma~\ref{lem:convolution-expansion}(i), 
	\begin{align*}
	q(y) = p(y) + \eta \Delta p(y) + O(\eta^2) \;,
	\end{align*}
	where $O(\eta^2)$ is uniform in $y$ under Assumption~\ref{assump:density-smoothness} for up to $k=4$.

	By Taylor expansion of $p(y + \eta \nabla f(y))$ up to second-order, we know 
	\begin{align*}
	p(y + \eta \nabla f(y)) = p(y) + \eta \langle \nabla p(y), \nabla f(y) \rangle + O(\eta^2) \;.
	\end{align*}
	Here the $O(\eta^2)$ term is uniform in $y$ under Assumption~\ref{assump:density-smoothness} for up to $k=2$ and Assumption~\ref{assump:bdd-hessian}(i).

	Moreover, by the expansion of the determinant, and Assumption~\ref{assump:bdd-hessian}(i), we have
	\begin{align*}
	\det( I_d + \eta \nabla^2 f(y) ) = 1 + \eta \mathrm{tr}(\nabla^2 f(y)) + O(\eta^2) = 1 + \eta \Delta f(y) + O(\eta^2) \;.
	\end{align*}
	Again, the $O(\eta^2)$ term is uniform in $y$ ensured by Assumption~\ref{assump:bdd-hessian}(i).

	Put all three terms together, and recall the uniform boundedness of $p, \nabla^2 f$, we have
	\begin{align*}
	d_{\mathrm{MA}}( \mathbf{T}_1^{f} \sharp P_Y, P_X ) 
	&\leq \sup_{y \in \Omega}  \eta  \big|   \Delta p(y) - \langle \nabla p(y), \nabla f(y) \rangle -  p(y) \Delta f(y) \big|  + O(\eta^2) \;. 
	\end{align*}
	where the $O(\eta^2)$ term is uniform in $y$. Recall that $f = f_*$ satisfies 
	$$
	 \Delta q - \langle \nabla q, \nabla f \rangle -  q \Delta f = 0 \;.
	$$
	By Lemma~\ref{lem:convolution-expansion}(i), we have
	\begin{align*}
		\Delta q = \Delta p + O(\eta) \;,\quad 
		\nabla q = \nabla p + O(\eta) \;,\quad
		q = p + O(\eta) \;.
	\end{align*}
	Together with the assumption that $\nabla f, \Delta f$ are uniformly bounded in $y$ (Assumption~\ref{assump:bdd-hessian}(i)), we conclude
	\begin{align*}
		\big|   \Delta p(y) - \langle \nabla p(y), \nabla f(y) \rangle -  p(y) \Delta f(y) \big| = O(\eta) \;,
	\end{align*}
	where the $O(\eta)$ term is uniform in $y$.
	
	Putting everything together, we have
	\begin{align*}
	d_{\mathrm{MA}}( \mathbf{T}_1^{f} \sharp P_Y, P_X ) = O(\eta^2) = O(\sigma^4) \;.
	\end{align*}
\end{proof}

\begin{proof}[Proof of Corollary~\ref{cor:distributional-matching-bayes-optimal}]
As in the proof of Theorem~\ref{thm:distributional-matching-first-order}, we have
\begin{align*}
d_{\mathrm{MA}}( \mathbf{T}_1^{f} \sharp P_Y, P_X ) 
	&\geq \sup_{y \in \Omega}    \big|  \eta ( \Delta p(y) - \langle \nabla p(y), \nabla f(y) \rangle -  p(y) \Delta f(y) )   \big| -  O(\eta^2)  \;, \\
	& \geq \sup_{y \in \Omega}    \big|  \eta ( \Delta q(y) - \langle \nabla q(y), \nabla f(y) \rangle -  q(y) \Delta f(y) )   \big| -  O(\eta^2)  \;.
\end{align*}
Note that the Bayes-optimal denoiser corresponds to the choice $f(y) = 2\log q(y)$, and we complete the proof.
\end{proof}

\begin{proof}[Proof of Theorem~\ref{thm:distributional-matching-second-order}]
	The proof continues the footprint of Theorem~\ref{thm:distributional-matching-first-order}. By definition, we have
	\begin{align*}
	d_{\mathrm{MA}}( \mathbf{T}_2^{f,g} \sharp P_Y, P_X ) &= \sup_{y \in \Omega} \left| q(y) - p(y + \eta \nabla f(y) + \tfrac{\eta^2}{2} \nabla g(y)) \det( I_d + \eta \nabla^2 f(y) + \tfrac{\eta^2}{2} \nabla^2 g(y) ) \right| \;.
	\end{align*}

	We deal with each term separately. First, by Lemma~\ref{lem:convolution-expansion}(ii), we have
	\begin{align*}
	q(y) = p(y) + \eta \Delta p(y) + \frac{\eta^2}{2} \Delta^2 p(y) + O(\eta^3) \;,
	\end{align*}
	where $O(\eta^3)$ is uniform in $y$ under Assumption~\ref{assump:density-smoothness} for up to $k=6$.

	By Taylor expansion of $p(y + \eta \nabla f(y) + \tfrac{\eta^2}{2} \nabla g(y))$ up to third-order, we know
	\begin{align*}
	p(y + \eta \nabla f(y) + \tfrac{\eta^2}{2} \nabla g(y)) &= p(y) + \eta \langle \nabla p(y), \nabla f(y) \rangle \\
		&\quad + \frac{\eta^2}{2} \left[ \langle \nabla p(y), \nabla g(y) \rangle + \langle \nabla^2 p(y), \nabla f(y) \otimes \nabla f(y) \rangle \right] + O(\eta^3) \;.
	\end{align*}
	Here the $O(\eta^3)$ term is uniform in $y$ under Assumption~\ref{assump:density-smoothness} for up to $k=3$, and Assumption~\ref{assump:bdd-hessian}(i) \& (ii).

	For the determinant term, by Lemma~\ref{lem:determinant-expansion}, we have
	\begin{align*}
	\det( I_d + \eta \nabla^2 f(y) + \tfrac{\eta^2}{2} \nabla^2 g(y) ) &= 1 + \eta \Delta f(y) + \frac{\eta^2}{2} \left[ \Delta g(y) + (\Delta f(y))^2 - \langle \nabla^2 f(y), \nabla^2 f(y) \rangle \right] + O(\eta^3) \;.
	\end{align*}
	Again, the $O(\eta^3)$ term is uniform in $y$ ensured by Assumption~\ref{assump:bdd-hessian}(i) \& (ii).

	Plug in the expression for all three terms, and the uniform boundedness of $p, \nabla^2 f, \nabla^2 g$, we have
	\begin{align*}
	&d_{\mathrm{MA}}( \mathbf{T}_2^{f,g} \sharp P_Y, P_X ) \leq \sup_{y \in \Omega}  \left| \eta \cdot \mathcal{E}_1[p, f]   + \frac{\eta^2}{2} \cdot \mathcal{E}_2[p, f, g] \right| + O(\eta^3) \;,
	\end{align*}
	with an $O(\eta^3)$ uniform in $y$.
	Here
	\begin{align}
		\mathcal{E}_1[p, f] &:= \Delta p - \langle \nabla p, \nabla f \rangle -  p \Delta f  \;,\\
		\mathcal{E}_2[p, f, g] &:= \Delta^2 p - \big(\langle \nabla p, \nabla g \rangle + \langle \nabla^2 p, \nabla f \otimes \nabla f \rangle \big) -2 \langle \nabla p, \nabla f \rangle \Delta f - p \big(\Delta g + (\Delta f)^2 - \| \nabla^2 f \|^2 \big) \;.
	\end{align}
	Recall that in the proof of Theorem~\ref{thm:distributional-matching-first-order}, we have shown that $\mathcal{E}_1[q, f] = 0$ uniformly in $y$ when $f = f_*$. Note also by Lemma~\ref{lem:convolution-expansion}(ii), we have
	\begin{align*}
	\Delta q = \Delta p + \eta \Delta^2 p + O(\eta^2) \;, \quad \nabla q = \nabla p + \eta \nabla \Delta p + O(\eta^2)  \;, \quad q = p + \eta \Delta p + O(\eta^2) \;.
	\end{align*}
	and therefore
	\begin{align*}
		\mathcal{E}_1[p, f] = \mathcal{E}_1[q, f] - \eta \cdot \Big[ \Delta^2 p - \langle \nabla \Delta p, \nabla f \rangle - \Delta p \Delta f \Big] + O(\eta^2) \;.
	\end{align*}
	Insert it back to the upper bound of $d_{\mathrm{MA}}( \mathbf{T}_2^{f,g} \sharp P_Y, P_X )$, and note $\mathcal{E}_1[q, f] = 0$
	\begin{align*}
		\eta \cdot \mathcal{E}_1[p, f]   + \frac{\eta^2}{2} \cdot \mathcal{E}_2[p, f, g] = \frac{\eta^2}{2} \cdot \Big[ \mathcal{E}_2[p, f, g] - 2 \big( \Delta^2 p - \langle \nabla \Delta p, \nabla f \rangle - \Delta p \Delta f \big) \Big] + O(\eta^3) \;.
	\end{align*}
	Define
	\begin{align}
		\mathcal{E}_2^{\mathrm{eff}}[p, f, g] &:= \mathcal{E}_2[p, f, g] - 2 \big( \Delta^2 p - \langle \nabla \Delta p, \nabla f \rangle - \Delta p \Delta f \big) \;. 
	\end{align}
	We will first verify that $\mathcal{E}_2^{\mathrm{eff}}[q, f, g] = 0$ when $f = f_*$ and $g = g_*$, and then use Lemma~\ref{lem:convolution-expansion}(ii) to conclude that $\mathcal{E}_2^{\mathrm{eff}}[p, f, g] -  \mathcal{E}_2^{\mathrm{eff}}[q, f, g] = O(\eta)$ uniformly in $y$ to complete the proof
	\begin{align}
	d_{\mathrm{MA}}( \mathbf{T}_2^{f,g} \sharp P_Y, P_X ) &\leq \sup_{y \in \Omega}  \frac{\eta^2}{2} \cdot \left| \mathcal{E}_2^{\mathrm{eff}}[p, f, g] - \mathcal{E}_2^{\mathrm{eff}}[q, f, g] \right| + O(\eta^3) = O(\eta^3) = O(\sigma^6) \;.
	\end{align}

	All we left is to establish $\mathcal{E}_2^{\mathrm{eff}}[q, f, g] = 0$. Plug in the expression of $\mathcal{E}_2[q, f, g]$, we have
	\begin{align*}
		\mathcal{E}_2^{\mathrm{eff}}[q, f, g] &=  - \big[\langle \nabla q, \nabla g \rangle + q \Delta g \big] - \big[ \langle \nabla^2 q, \nabla f \otimes \nabla f \rangle  +  2 \langle \nabla q, \nabla f \rangle \Delta f   + q[ (\Delta f)^2 - \langle \nabla^2 f, \nabla^2 f \rangle ]\big] \\
		& \quad \quad - \big[ \Delta^2 q - 2\langle \nabla \Delta q, \nabla f \rangle - 2 \Delta q \Delta f    \big] \;.\\
		& = - \nabla \cdot ( q \nabla g ) - \nabla \cdot ( \Delta q \nabla f - \nabla^2 f \nabla q ) + \nabla \cdot (2  \Delta q \nabla f - \nabla \Delta q ) \;,
	\end{align*}
	where the last line uses Lemma~\ref{lem:third-order-identity} for the middle bracketed term.
	Recall that $f = f_*$ and $g = g_*$ satisfy
	\begin{align*}
	 q \nabla g  &= \Delta q \nabla f + \nabla^2 f \nabla q - \nabla \Delta q \;,
	\end{align*}
	using Lemma~\ref{lem:log-density-identities}.
	Therefore, we conclude that $\mathcal{E}_2^{\mathrm{eff}}[q, f, g] = 0$ when $f = f_*$ and $g = g_*$.
\end{proof}

\section{Supporting Lemmas}
\label{sec:proofs-of-supporting-lemmas}
\begin{lemma}
	\label{lem:convolution-expansion}
	Let $X \in \mathbb{R}^d$ be a random variable with density $p: \mathbb{R}^d \rightarrow \mathbb{R}_{\geq 0}$, and $Z \in \mathbb{R}^d$ be a random variable with density $\phi: \mathbb{R}^d \rightarrow \mathbb{R}_{\geq 0}$. Consider the noisy measurement $Y = X + \sqrt{2\eta} Z$ with density $q: \mathbb{R}^d \rightarrow \mathbb{R}_{\geq 0}$.

	(i) Let Assumption~\ref{assump:noise-moments}(i) hold, and Assumption~\ref{assump:density-smoothness} holds for up to $k=4$. Then the density $q$ admits the following expansion
	\begin{align*}
	q(y) = p(y) + \eta \Delta p(y) +  O(\eta^{2}) \;.
	\end{align*}
	Moreover, the derivative of $q(y)$ admits the expansion
	\begin{align*}
	\partial_i q(y) &= \partial_i p(y) + O(\eta) \;,\\
	\Delta q(y) &= \Delta p(y) + O(\eta) \;.
	\end{align*}

	(ii) Let Assumptions~\ref{assump:noise-moments}(i) \& (ii) hold, and let Assumption~\ref{assump:density-smoothness} hold for up to $k=6$, then
	\begin{align*}
	q(y) = p(y) + \eta \Delta p(y) + \frac{\eta^2}{2} \Delta^2 p(y) + O( \eta^{3}) \;.
	\end{align*}
	Moreover, the derivatives of $q(y)$ admit the expansion
	\begin{align*}
	\partial_i q(y) &= \partial_i p(y) + \eta \partial_i \Delta p(y) + O(\eta^{2}) \;, \\
	\Delta q(y) &= \Delta p(y) + \eta \Delta^2 p(y) + O(\eta^{2}) \;, \\
	\partial_i \Delta q(y) &= \partial_i \Delta p(y) + O(\eta) \;.
	\end{align*}

	Here, the $O(\cdot)$ terms are uniform over $y$.
\end{lemma}
\begin{proof}[Proof of Lemma~\ref{lem:convolution-expansion}]
	By the definition of convolution, we have
	\begin{align}
		\label{eqn:convolution-density}
	q(y) &= \int_{\mathbb{R}^d} p(y - \sqrt{2\eta} z) \phi(z) \dd{z} \;.
	\end{align}
	For a vector $z \in \mathbb{R}^d$, we denote $z^{\otimes k} := z \otimes z \otimes \cdots \otimes z$ ($k$ times) as the $k$-th order tensor power of $z$.
	Using Taylor expansion of $p(y - \sqrt{2\eta} z)$ around $y$, we know there exists some $\delta = \delta(y, \eta, z, p) \in [0,1]$ such that
	\begin{align*}
	p(y - \sqrt{2\eta} z) &= p(y) - \sqrt{2\eta} \langle \nabla p(y), z \rangle + \frac{(-\sqrt{2\eta})^2}{2!} \langle \nabla^2 p(y), z^{\otimes 2} \rangle + \frac{(-\sqrt{2\eta})^3}{3!} \langle \nabla^3 p(y), z^{\otimes 3} \rangle \\
	& \quad + \frac{(-\sqrt{2\eta})^4}{4!} \langle \nabla^4 p(y+\delta(-\sqrt{2\eta} z)), z^{\otimes 4}\rangle \;.
	\end{align*}
	Plugging the above expansion into the convolution integral,
	\begin{align*}
	&\left| q(y) - \left[ p(y) + \frac{(\sqrt{2\eta})^2}{2!} \langle \nabla^2 p(y), \E[Z^{\otimes 2}] \rangle \right] \right| \\
	&\quad \leq  \int_{\mathbb{R}^d} \frac{(\sqrt{2\eta})^4}{4!} \langle |\nabla^4 p(y{+}\delta(-\sqrt{2\eta} z)) |, |z^{\otimes 4} | \rangle  \phi(z) \dd{z} = C \cdot \eta^{2} \;,
	\end{align*}
	where the last equality follows from Assumption~\ref{assump:density-smoothness} for $k=4$, with a constant depending only on $\| \nabla^4 p \|_{\infty}$ and the fourth moment of $Z$. The $O(\eta^{2})$ is uniform over $y$ and $\eta \in [0, 1)$. Recall $\langle \nabla^2 p(y), \E[Z^{\otimes 2}] \rangle = \Delta p(y)$ by Assumption~\ref{assump:noise-moments}(i), the first part of the lemma is proved.

	For derivatives $\partial_i q$, we can take derivative $\partial_i$ inside \eqref{eqn:convolution-density} (ensured by dominated convergence), the convolution integral defining $q(y)$, and repeat the above Taylor expansion argument for $\partial_i p(y-\sqrt{2\eta} z)$ up to the second-order terms, with the remainder term bounded by $C \cdot \eta$ for some constant $C$ depending only on $\| \nabla^3 p \|_{\infty}$ and the third absolute moment of $Z$. The same logic applies to the convolution integral defining $\Delta q(y)$, with expansion applied to $\Delta p(y - \sqrt{2\eta} z)$. Note that the remainder term is bounded by $C \cdot \eta$ for some constant $C$ depending only on $\| \nabla^2 \Delta p \|_{\infty}$ and the fourth moment of $Z$.

	For the second part of the lemma, we further expand the Taylor expansion to the $5$-th and $6$-th order terms, namely, there exists some $\delta = \delta(y, \eta, z, p) \in [0,1]$ such that
	\begin{align*}
	&\left| q(y) - \left[ p(y) + \eta \Delta p(y) + \frac{(\sqrt{2\eta})^4}{4!} \langle \nabla^4 p(y), \E[Z^{\otimes 4}] \rangle \right] \right| \\
	&\quad \leq  \int_{\mathbb{R}^d} \frac{(\sqrt{2\eta})^6}{6!} \langle |\nabla^6 p(y{+}\delta(-\sqrt{2\eta} z)) |, |z^{\otimes 6} | \rangle  \phi(z) \dd{z} \;,
	\end{align*}
	where the last term is bounded by $C \cdot \eta^{3}$ for some constant $C$ depending only on $\| \nabla^6 p \|_{\infty}$ and the sixth moment of $Z$. By Assumption~\ref{assump:noise-moments}(ii), we can verify that
	$$\langle \nabla^4 p(y), \E[Z^{\otimes 4}] \rangle = 3 \cdot \Delta^2 p(y)\;.$$ 

	For the derivatives $\partial_i q$. We can take the derivative $\partial_i$ inside \eqref{eqn:convolution-density}, and repeat the above Taylor expansion argument up to the fourth-order terms, with the remainder term bounded by $C \cdot \eta^{2}$ for some constant $C$ depending only on $\| \nabla^5 p \|_{\infty}$ and the fifth absolute moment of $Z$. $\Delta q$ term can be handled similarly by taking derivative $\Delta$ inside the convolution integral, and expand $\Delta q(y-\sqrt{2\eta}z)$, with remainder term depending on $\|\nabla^4 \Delta p \|_\infty$, and thus bounded by $\|\nabla^6 p \|_\infty$. The expansion of $\partial_i \Delta q(y)$ can be proved similarly by taking derivative $\partial_i \Delta$ inside \eqref{eqn:convolution-density} and repeating the Taylor expansion argument up to the second-order terms, with the remainder term bounded by $C \cdot \eta$ for some constant $C$ depending only on $\| \nabla^5 p \|_{\infty}$ and the third absolute moment of $Z$.
\end{proof}

\begin{lemma}
	\label{lem:log-density-identities}
	The following identity holds for $f = \log q$.
	\begin{enumerate}
	\item[(i)] $q \nabla f = \nabla q$
	\item[(ii)] $q (\nabla^2 f + \nabla f \otimes \nabla f) = \nabla^2 q$
	\item[(iii)] $q (\Delta f + \| \nabla f \|^2) = \Delta q $
	\item[(iv)] $q ( \Delta f \nabla f + \|\nabla f\|^2 \nabla f + \nabla \Delta f + 2 \nabla^2 f \nabla f) = \nabla \Delta q$
	\end{enumerate}
\end{lemma}
\begin{proof}[Proof of Lemma~\ref{lem:log-density-identities}]
(i) follows directly from the definition of $f = \log q$.
(ii) follows by taking derivative $\nabla$ on both sides of (i), and replace $\nabla q$ by $q \nabla f$ as in (i). (iii) is the trace of (ii). (iv) follows by taking derivative $\nabla$ on both sides of (iii), and replace $\nabla q$ by the expressions in (i).
\end{proof}

\begin{lemma}
	\label{lem:third-order-identity}
	Let $f = \log q$, we have
	\begin{align*}
		\nabla \cdot \left( \Delta q \nabla f - \nabla^2 f \nabla q \right) = \langle \nabla^2 q, \nabla f \otimes \nabla f \rangle  +  2 \langle \nabla q, \nabla f \rangle \Delta f   + q[ (\Delta f)^2 - \langle \nabla^2 f, \nabla^2 f \rangle ]	 \;.
	\end{align*}
\end{lemma}
\begin{proof}[Proof of Lemma~\ref{lem:third-order-identity}]
By the chain rule, we have
\begin{align*}
&  \nabla \cdot \left( \Delta q \nabla f - \nabla^2 f \nabla q \right) = \langle \nabla \Delta q, \nabla f\rangle + \Delta q \Delta f - \langle \nabla \Delta f, \nabla q \rangle - \langle \nabla^2 f, \nabla^2 q \rangle 
\end{align*}
We deal with each term separately on the right-hand side.

First, 
\begin{align*}
 \langle \nabla \Delta q, \nabla f\rangle = q (\Delta f + \|\nabla f\|^2) \| \nabla f \|^2 + q \langle \nabla \Delta f,  \nabla f \rangle + 2 q \langle \nabla^2 f,  \nabla f \otimes \nabla f \rangle ~~\text{use Lemma~\ref{lem:log-density-identities}(iv)} \; .
\end{align*}
Second,
\begin{align*}
 \Delta q \Delta f = q (\Delta f + \|\nabla f\|^2) \Delta f ~~\text{use Lemma~\ref{lem:log-density-identities}(iii)} \;.
\end{align*}
Third
\begin{align*}
	- \langle \nabla \Delta f, \nabla q \rangle &= - q \langle \nabla \Delta f, \nabla f \rangle ~~\text{use Lemma~\ref{lem:log-density-identities}(i)} \;,\\
	-\langle \nabla^2 f, \nabla^2 q \rangle &= - q \langle \nabla^2 f, \nabla^2 f \rangle - q \langle \nabla^2 f, \nabla f \otimes \nabla f \rangle ~~\text{use Lemma~\ref{lem:log-density-identities}(ii)} \;.
\end{align*}

Combining the above four displays, we know
\begin{align*}
& \langle \nabla \Delta q, \nabla f\rangle + \Delta q \Delta f - \langle \nabla \Delta f, \nabla q \rangle - \langle \nabla^2 f, \nabla^2 q \rangle \\
&= q \langle \nabla^2 f, \nabla f \otimes \nabla f \rangle  - q \langle  \nabla^2 f, \nabla^2 f \rangle + q(\Delta f + \|\nabla f\|^2)^2 \\
&= q \langle \nabla^2 f +  \nabla f \otimes \nabla f, \nabla f \otimes \nabla f \rangle + 2 q \langle \nabla f, \nabla f \rangle \Delta f + q (\Delta f)^2 - q \langle \nabla^2 f, \nabla^2 f \rangle \\
&\text{recall Lemma~\ref{lem:log-density-identities}(ii), and (i), we have} \\
& = \langle \nabla^2 q, \nabla f \otimes \nabla f \rangle + 2 \langle \nabla q, \nabla f \rangle \Delta f + q \left[ (\Delta f)^2 -  \langle \nabla^2 f, \nabla^2 f \rangle \right] \;.
\end{align*}

\end{proof}

Finally, we need a standard expansion of the determinant function, which can be verified directly.
\begin{lemma}
	\label{lem:determinant-expansion}
	For a symmetric matrix $A = A^\top$, we have
	$$\det( I_d + \eta A ) = 1 + \eta \tr(A) + \frac{\eta^2}{2} \left( \tr(A)^2 - \tr(A^2) \right) + O(\eta^3) \;.$$
\end{lemma}
\begin{proof}[Proof of Lemma~\ref{lem:determinant-expansion}]
	Let $\lambda_1, \lambda_2, \ldots, \lambda_d$ be the eigenvalues of $A$, all bounded. Then
	\begin{align*}
	\det( I_d + \eta A ) = \prod_{i=1}^d (1 + \eta \lambda_i) = 1 + \eta \sum_{i=1}^d \lambda_i + \frac{\eta^2}{2} \left[ \left( \sum_{i=1}^d \lambda_i \right)^2 - \sum_{i=1}^d \lambda_i^2 \right] + O(\eta^3) \;.
	\end{align*}
	Noting that $\tr(A) = \sum_{i=1}^d \lambda_i$ and $\tr(A^2) = \sum_{i=1}^d \lambda_i^2$, we complete the proof.
\end{proof}

\end{document}